\documentclass[twoside]{article}
\usepackage{amssymb}
\usepackage{amsmath}
\usepackage{amsthm}
\usepackage{algorithm}
\usepackage{algpseudocode}
\usepackage{caption}
\usepackage{subcaption}
\usepackage{enumerate}
\usepackage{pgfplots}
\usepackage{multicol}
\usepackage[hmarginratio=1:1,top=32mm,columnsep=20pt]{geometry}
\usepackage{fullpage}
\usepackage{pdflscape}
\usepackage{setspace}
\usepackage{url}
\usepackage{tikz}
\usetikzlibrary{tikzmark}
\usepackage{color}
\usepackage{pgfplots}
\usepackage{bm}
\usepackage{tikz-cd}
\pgfplotsset{compat=newest}
\usepgfplotslibrary{fillbetween}
\usepackage[toc,page]{appendix}

\usetikzlibrary{arrows.meta,
                calc, chains,
                patterns,
                patterns.meta,
                fit,
                quotes,
                positioning,
                shapes.geometric,
                shapes.misc,
                automata
                }
\tikzset{
   recbox/.style = {
         rectangle,
         draw, 
         align = center, 
         text badly centered,
         inner sep = 6 pt,
         font=\footnotesize,
         line width = 0.3mm,
      },
      circlebox/.style = {
         rounded rectangle,
         draw, 
         align = center, 
         text badly centered,
         inner sep = 7 pt,
         font=\large,
         line width = 0.5mm,
      },
      roundbox/.style = {
         rectangle,
         draw, 
         align = center, 
         rounded corners,
         text badly centered,
         inner sep = 6 pt,
         font=\large,
         line width = 0.5mm,
      },
     box1/.style = {
         rectangle,
         draw, 
         align = center, 
         text badly centered,
         inner sep = 6 pt,
         font=\large,
         line width = 0.5mm,
         minimum width = 30mm,
         minimum height = 7mm,
      },
    papLine/.style = {
         draw,
         -stealth,
         font=\ttfamily,
         line width = 0.5mm,
      },
      }

\definecolor{darkgreen}{rgb}{0.01, 0.75, 0.24}
\definecolor{shaded_green}{RGB}{151, 194, 157}

\theoremstyle{plain}
\newtheorem{theorem}{Theorem}[section]
\newtheorem{lemma}[theorem]{Lemma}
\newtheorem{proposition}[theorem]{Proposition}

\theoremstyle{definition}
\newtheorem{definition}{Definition}[section]
\newtheorem{example}{Example}[section]

\theoremstyle{remark}

\begin{document}
\parindent=0in
\parskip=12pt


\title{
  Relative Probability on Finite Outcome Spaces \\
  \large{
    A Systematic Examination of its Axiomatization, Properties, and Applications
  }
}

\author{Max Sklar\\ Local Maximum Labs\\ December 30, 2022}
\date{}

\maketitle
\thispagestyle{empty}

\begin{abstract}
This work proposes a view of probability as a relative measure rather than an absolute one. To demonstrate this concept, we focus on finite outcome spaces and develop three fundamental axioms that establish requirements for relative probability functions. We then provide a library of examples of these functions and a system for composing them. Additionally, we discuss a relative version of Bayesian inference and its digital implementation. Finally, we prove the topological closure of the relative probability space, highlighting its ability to preserve information under limits.
\end{abstract}

\tableofcontents
\newpage

\section{Introduction}

The foundations of probability theory are still very much open to explore!

Since Kolmogorov published the standard axioms for probability\cite{kolmogorov} in 1933, there have been calls to alter them. In ``Kolmogorov's Axiomatisation and Its Discontents'', Lyon\cite{lyon} examines arguments made in favor and against these alterations. One area of ``discontentment'' concerns conditional probability. We often want to identify the probability of event A given event B, or \(P(A|B)\), even when \(B\) has probability zero\footnote{Note that zero probability events can indeed occur, particularly when given a continuous distribution. November\cite{november} gave a more recent philosophical treatment of this phenomenon.}.

We are out of luck with the Kolmogorov model, which defines \(P(A|B)\) as the ratio \(\frac{P(A \cap B)}{P(B)}\). When \(P(B) = 0\), the indeterminate form \(\frac{0}{0}\) appears leaving the conditional probability undefined. In fact, this happens whenever one wishes to compare two events that have probability zero.

Undeterred by the petty obstacle of the indeterminate form, mathematicians and engineers refer to relative probabilities of this type all the time. For example, if we consider a probability distribution over \([0, 1]\) given by \(f(x) = 2x\), we know that the value of \(f(\frac{1}{2}) = 1\) is twice as much as \(f(\frac{1}{4}) = \frac{1}{2}\). In a sense, we believe that the outcome \(\frac{1}{2}\)  is twice as likely as outcome \(\frac{1}{4}\) even though we are only talking about \textit{probability density}.

Hajek\cite{hajek} (citing Borel) gives a much more compelling example: if a random point on the Earth is selected, what is the probability that it is in the eastern hemisphere given that it is on the equator? It seems that one should not hesitate to answer one half. And yet the equator, being a mere 1-dimensional object, has probability zero compared to the rest of the globe.

To address the discrepancy between theory and practice, we propose a non-standard model of probability. This model is built on the relationships between outcomes and events rather than their absolute likelihoods. By focusing on these relationships, this improves on the Kolmogorov model by solving the conditional probability question. As additional benefits, it fits nicely with most distribution sampling techniques and establishes probability functions as categories.

\subsection{Previous Work and Goals}

Leading probability theorists in the twentieth century developed axiomatic systems for conditional probability, notably Renyi\cite{renyi}. Later, Kohlberg and Reny\cite{kohlberg} introduced the idea of relative probability and applied it to game theory. Heinemann\cite{heinemann} further developed the idea into \textit{relative probability measures} along with their axioms and definitions, and found applications in Bayesian inference and economics\cite{heinemann_econ}.

This work is focused on the systematization of these concepts and will expand on them in several ways.

First, we will construct a theory of relative probability on finite outcome spaces. By omitting infinite outcome spaces, we temporarily set aside the need to account for measurable sets and countable additivity\footnote{In the textbook Invitation to Discrete Mathematics\cite{discrete}, Matoušek et al. make the same simplification in order to illustrate concepts while at the same time writing \quote{By restricting ourselves to finite probability spaces we have simplified the situation considerably... A true probability theorist would probably say that we have excluded everything interesting.}}. Even with this vast simplification there is much to be learned. The fundamental definitions for relative probability will be more clearly constructed without the distractions that are introduced by continuous space.

Our focus on probability at the outcome level will lead us to separate out three fundamental axioms of relative probability. These fundamental axioms, acting on outcomes only, will be distinct from those related to summation and measurability. They will establish relative probability as a \textit{thin category} on outcomes.

Second, we will categorize the various patterns and states that arise when dealing with relative probability, including total comparability, possibility classes, and anchor outcomes. 

Third, we focus on the computational properties of Bayesian inference. The relative probability function will simplify the formulas for some distributions in the Bayesian framework. We introduce the indeterminate wildcard value, which will inevitably arise when relative probability is inferred on messy, real world data. This discussion ends with how these concepts can be applied to code and data structures.

Finally, we prove the ability of relative probability functions to retain information when taking limits, or in other words their topological closure.

Ultimately, practitioners will find these features of relative probability attractive. Its sphere of application could be expanded beyond game theory and theoretical probability into fields like machine learning. This paper provides a foundational analysis that future researchers can use to expand that sphere.

\section{Preliminaries}
\subsection{Magnitude Space}

\begin{definition}
The \textit{magnitude space} \(\mathbb{M}\) is the set of all positive real numbers along with \(0\) and \(\infty\).
\[\mathbb{M} = [0, +\infty]\]
\end{definition}

Magnitudes correspond with our intuition of size. The value of infinity is a \textit{limit element}, larger that all of the other magnitudes. It endows the magnitude space with several important properties:

\begin{enumerate}
\item Compactness: Sequences that go off to infinity still have a limit (at \(\infty\)).
\item Symmetry around ratios: When we compare the probability of two events, we get their \textit{odds}. If the odds are 0, then we are comparing an event with probability 0 to an event with probability \(>0\). We should be able to reverse this comparison, and say there are infinite odds when an event with probability \(>0\) is compared to an event with probability 0. It is also common to define the odds of a single event as the odds of that event against its converse. In this case, \(\infty\) corresponds to events that are certain.
\item The infinite element is introduced in measure theory because many mathematical systems (real and natural numbers for example) contain sets of infinite measure.
\end{enumerate}

We set \(0^{-1} = \infty\) and \(\infty^{-1} = 0\), even though the product \(0 \cdot \infty\) is indeterminate.

\subsection{The Wildcard Element}

\begin{definition}
Let the \textit{magnitude-wildcard space} \(\mathbb{M}^*= \mathbb{M} \cup \{\ast\}\) be the set of magnitudes along with a \textit{wildcard element}, \(\ast\).
\end{definition}

The wildcard element corresponds to several different concepts, each appearing in a unique discipline:
\begin{itemize}
  \item The \textit{NaN}, or \textit{Not a Number}\footnote{``Not a Number'' may have been an unfortunate naming choice because it actually represents \textbf{any} number!} value in the IEEE standard for floating point arithmetic\cite{ieee}.
  \item The indeterminate form \(\frac{0}{0}\) in arithmetic.
  \item The \textit{wildcard pattern} used in pattern matching and regular expressions in type theory and computer science
\end{itemize}

The following properties on \(\ast\) allow multiplication of any two magnitude-wildcard values.
\[0 \cdot \infty = \ast \qquad \ast \cdot m = \ast\]

We may also define \(\ast + m = \ast\), but there is an argument to  be made that because there are no negative numbers in this system, \(\ast + m\) should be a new pattern that only matches magnitudes greater than or equal to \(m\). For our purposes, the question of how to handle this can be deferred.

\subsection{The Matching Relation}

\begin{definition}
The \textit{matching relation} \(:\cong\) is a binary relation on \(\mathbb{M}^*\). The statement \(m_1 :\cong m_2\) is read ``\(m_1\) is matched by \(m_2\)'' and is true when either \(m_1\) equals \(m_2\) or \(m_2\) is a wildcard.
\[m_1 :\cong m_2 \Longleftrightarrow (m_1 = m_2) \vee (m_2 = \ast)\]
\end{definition}

The left hand side of a matching relation is the \textit{parameter} and the right hand side is the \textit{constraint}. The wildcard element represents every single value, but it cannot be represented by any specific value. It also represents a loss of information about the parameter.

A few lemmas quickly follow.

\begin{lemma}
\label{wild_prop_1} If a magnitude matches a non-wildcard element, then the two values are equal. \[m_1 :\cong m_2 \wedge m_2 \neq \ast \Longrightarrow m_1 = m_2\]
\end{lemma}

\begin{lemma}
\label{wild_prop_2}
Every element is matched by the wildcard element. \(m :\cong \ast\)
\end{lemma}

\begin{lemma}
\label{wild_prop_3}
The wildcard element is matched only by itself. \(\ast :\cong m \Longrightarrow m = \ast\)
\end{lemma}

The matching relation looks a lot like equality and in many cases it is, but it doesn't have all of the same properties.

\begin{theorem}
The matching relation is reflexive, transitive, and anti-symmetric but not symmetric.
\end{theorem}

\begin{proof}
Reflexive is obvious: \(m :\cong m \Longleftrightarrow (m = m) \vee (m = \ast)\)

The transitive property states that for all \(m_1, m_2, m_3\) in \(\mathbb{M}\), if \(m_1 :\cong m_2\) and \(m_2 :\cong m_3\), then \(m_1 :\cong m_3\).

Assume that \(m_1 :\cong m_2\) and \(m_2 :\cong m_3\). If none of these values are the wildcards, then by lemma \ref{wild_prop_1}, they are all equal and \(m_1 :\cong m_3\). If \(m_1 = \ast\) then by lemma \ref{wild_prop_3}, \(m_2 = \ast\) and finally \(m_3 = \ast\). In other words, if any of the three values are \(\ast\), then \(m_3 = \ast\). By lemma \ref{wild_prop_2}, the theorem holds.

The anti-symmetric property states that if \(m_1 :\cong m_2\) and \(m_2 :\cong m_1\), then \(m_1 = m_2\).

If both \(m_1\) and \(m_2\) match each other, then \((m_1 = m_2) \vee (m_2 = \ast)\) and \((m_2 = m_1) \vee (m_1 = \ast)\). By boolean factorization, this means that \((m_1 = m_2) \vee ((m_2 = \ast) \wedge (m_1 = \ast))\). On both sides of the conjunction, \(m_1 = m_2\).

The matching relation is not symmetric because \(m :\cong \ast\) but \(\ast :\ncong m\) for any non-wildcard \(m\).
\end{proof}

\section{Categorical Distributions}

Let \(\Omega\) be a set of mutually exclusive \textit{outcomes}\footnote{Each outcome could be thought of as a possible result of a random trial, or a possible value for an unknown variable}. We assume that \(\Omega\) is finite and there are \(K\) outcomes or \textit{categories} so that \(|\Omega| = K\).

\begin{definition}
\label{def:categorical_abs}
A \textit{categorical distribution} on a \(\Omega\) is a function \(P: \Omega \rightarrow [0, 1]\) such that \(\sum_{h \in \Omega} P(h) = 1\)
\end{definition}

The set of all categorical distributions of size \(K\) can be embedded in \(\mathbb{R}^K\) as a (K-1)-dimensional object called a simplex (see figure \ref{fig:simplex}). For example, if \(K = 3\), the resulting space of categorical distributions is an equilateral triangle embedded in \(\mathbb{R}^3\) connecting the points (1, 0, 0), (0, 1, 0), and (0, 0, 1).

\begin{figure}[h]
\centering
\resizebox{0.3\textwidth}{!}{
\begin{tikzpicture} 
\draw[-latex,line width=0.5mm] (0,2)--(0,0) -- (3.5,0) node[right]{\(P_{0}\)};
\draw[-latex,line width=0.5mm] (0,0)--(0,3.5) node[above]{\(P_{1}\)};
\draw[red,line width=0.5mm] (2,0)--(0,2);
\draw[fill=black] (2,0) circle (2pt) node[black,below,scale=0.8] {(0,1)};
\draw[fill=black] (0,2) circle (2pt) node[black,left,scale=0.8] {(1,0)};
\end{tikzpicture}
}
\resizebox{0.3\textwidth}{!}{
\begin{tikzpicture} 
\draw[-latex,line width=0.5mm] (0,2,0)--(0,0,0) -- (4,0,0) node[right]{\(P_{0}\)};
\draw[-latex,line width=0.5mm] (0,0,0) -- (0,4,0) node[above]{\(P_{1}\)};
\draw[-latex,line width=0.5mm,dashed] (0,0,0) -- (0,0,-6.5) node[above right]{\(P_{2}\)};
\draw[fill=red!70!,fill opacity=0.5] (2,0)--(1.3,1.3)--(0,2)--cycle;
\draw[fill=black] (2,0) circle (2pt) node[black,below,scale=0.8] {(1,0,0)};
\draw[fill=black] (1.3,1.3) circle (2pt) node[black,below right,scale=0.8] {(0,0,1)};
\draw[fill=black] (0,2) circle (2pt) node[black, left,scale=0.8] {(0,1,0)};
\end{tikzpicture}
}
\resizebox{0.3\textwidth}{!}{
\begin{tikzpicture} 
\draw[line width=0.5mm] (0,3)--(-2,0)--(2,0)--cycle;
\draw[line width=0.5mm,dotted] (0,1.3)--(2,0);
\draw[line width=0.5mm,dotted] (0,1.3)--(-2,0);
\draw[line width=0.5mm,dotted] (0,1.3)--(0,3);
\draw[fill=black] (2,0) circle (2pt) node[black,below,scale=0.8] {(0,1,0,0)};
\draw[fill=black] (0,1.3) circle (2pt) node[black,right,rotate=0, scale=0.8] {(0,0,0,1)};
\draw[fill=black] (0,3) circle (2pt) node[black,above,scale=0.8] {(1,0,0,0)};
\draw[fill=black] (-2,0) circle (2pt) node[black,below,scale=0.8] {(0,0,1,0)};
\end{tikzpicture}
}
\caption{An illustration of the probability simplex for K = 2, 3, and 4. These objects are respectively, a segment embedded in \(\mathbb{R}^2\), an equilateral triangle embedded in \(\mathbb{R}^3\), and a normal tetrahedron embedded in \(\mathbb{R}^4\). We make no attempt to visualize the 4D space that contains the tetrahedron.}
\label{fig:simplex}
\end{figure}
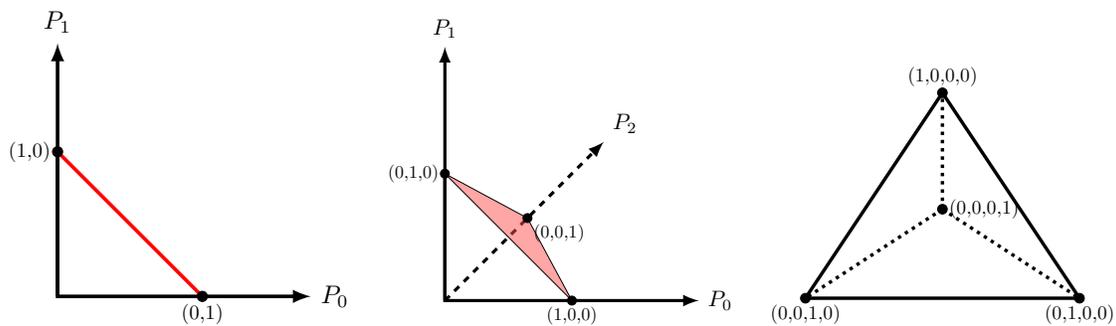

\subsection{Events}

An \textit{event} is a set of outcomes, and by convention \(\mathcal{F}\) is the set of all possible events. In general, \(\mathcal{F}\) is not the entire power set of \(\Omega\), but when \(\Omega\) is finite we can consider any subset \(e \subseteq \Omega\) to be an event\footnote{We need not concern ourselves with defining a \(\sigma\)-algebra of measurable sets.}.

Definition \ref{def:categorical_abs} creates a probability function that assigns values on individual outcomes. We now define the probability function on an event, which corresponds to the probability that any one of its outcomes occur. Looking at probability on the event level rather than the outcome level is a crucial insight in the development of probability theory (and measure theory more generally). Even though the process is far simpler for finite distributions, we must pay attention to this layer in order for the framework to generalize. 

\begin{definition}
For all \(e\) in \(\mathcal{F}\), \[P(e) = \sum_{h \in e}{P(h)}.\]
\end{definition}

\(P\) acts on either outcomes or events using the convention \(P(\{h\}) = P(h)\). \(\Omega\) is itself an event; the \textit{universal event} of all outcomes, with probability 1. \[P(\Omega) = \sum_{h \in \Omega}{P(h)} = 1\]

\subsection{Relative Probability Function}
\label{section:standard_relative_prob}

A \textit{relative probability function}, or \textit{RPF}, measures the probability of one event with respect to another. For example, we may wish to talk about an event that is ``twice as likely'' as another, even if we don't know the absolute probability of either event. In a sense, all probability is relative and conditional because all probabilistic statements come with underlying assumptions.

We continue to use P to represent the RPF but now with two inputs instead of one. The expression \(P(e_1, e_2)\) can be read as the probability of \(e_1\) relative to \(e_2\).

\[P: \mathcal{F} \times \mathcal{F} \rightarrow \mathbb{M}^*\]

We define relative probability in terms of absolute probability as a ratio, in the style of Kolmogorov.

\begin{definition}
\label{def:ratio}
The relative probability of events \(e_1\) and \(e_2\) on a categorical distribution \(P\) is as follows:
\[P(e_1, e_2) = \frac{P(e_1)}{P(e_2)}\]
\end{definition}

If \(P(e_1) = P(e_2) = 0\), then \(P(e_1, e_2) = \ast\), representing the motivational problem of zero-probability events being incomparable.

With absolute probability, information is lost at the boundaries of the simplex where the probability of several outcomes might be assigned a value of zero. For example, if \(\Omega = \{a, b, c\}\) with \(P(a) = 1\) and \(P(b) = P(c) = 0\), we cannot compare the probabilities of \(b\) and \(c\) by ratio as we can in the rest of the simplex.

This poses an interesting problem for limits.

\begin{example}
\label{ex:abs_lose_info}
Consider the following categorical distribution function, with parameter \(\epsilon > 0\):
\[
P(a) = 1 - \epsilon\qquad
P(b) = \frac{2}{3}\epsilon\qquad
P(c) = \frac{1}{3}\epsilon
\]
\end{example}

This is clearly an absolute probability, and its limit as \(\epsilon\) goes to zero should be \(P(a) = 1, P(b) = P(c) = 0\). The fact that b is twice as likely as c is lost!

\section{The Relative Probability Approach}
\label{section:new_relative_prob}

In section \ref{section:standard_relative_prob}, the relative probability function was derived from the absolute probability function. Here in section \ref{section:new_relative_prob}, we invert this process by starting with the relative probability as the fundamental building block.

\subsection{Fundamental Axioms}

\begin{definition}
\label{def:fundamental_laws}
Let \(\Omega\) be the set of outcomes, and \(P: \Omega \times \Omega \rightarrow \mathbb{M}^*\) be a function acting on two outcomes to produce a magnitude-wildcard. \(P\) is a \textit{relative probability function on the outcomes of \(\Omega\)} if it obeys the \textit{3 fundamental axioms of relative probability}:

\begin{enumerate}[(i)]
\item The \textit{identity axiom}: \(P(h, h) = 1\)
\item The \textit{inverse axiom}: \(P(h_1, h_2) = P(h_2, h_1)^{-1}\)
\item The \textit{composition axiom}: \(P(h_1, h_3) :\cong P(h_1, h_2) \cdot P(h_2, h_3)\)
\end{enumerate}

\end{definition}

\(P(h_1, h_2)\) represents the probability of \(h_1\) relative to \(h_2\). Outcomes \(h_1\) and \(h_2\) are \textit{comparable} if \(P(h_1, h_2) \neq \ast\).

Let us pause for a moment to discuss how these axioms were selected. The composition axiom is doing most of the work, and it succinctly encodes how relative probability works. If \(A\) is twice as likely as \(B\), and \(B\) is 3 times as likely as \(C\), then \(A\) had better be 6 times as likely as \(C\). If not, these relative probability assignments would have no meaning; they would just be numerical assignments without rhyme or reason\footnote{Many of our political and economic forecasts come in this form.}.

The composition axiom is enough to show that the identity axiom works most of the time. For any two outcomes \(h_1\) and \(h_2\) we get through composition \(P(h_1, h_2) :\cong P(h_1, h_1) \cdot P(h_1, h_2)\). As long as \(P(h_1, h_2)\) isn't \(0\), \(\infty\), or \(\ast\), then we would have to conclude \(P(h_1, h_1) = 1\) so long as \(h_1\) is comparable to itself.

But that doesn't get us all the way there! There are still scenarios where \(P(h, h) = \ast\). The question that must be asked is: should the the self-comparisons in an outcome space contain information where there is a choice of values between \(1\) and \(\ast\)? Such information would not be relevant to relative probability. Therefore, \(P(h, h)\) can only have a single value and it must be 1. Hence, the necessity of the identity axiom.

Composition and identity can actually be combined into a single axiom about composition paths. It's a bit unwieldy, but nevertheless interesting.

\begin{proposition}[Path Composition]
Given a non-empty list of \(N\) outcomes \(h_0, h_1, h_2, ..., h_{N-1}\), \[P(h_0, h_{N-1}) :\cong \prod_{k=0}^{N-2} P(h_k, h_{k+1}) \]
\end{proposition}

In this case, \(P(h_0, h_0)\) would be matched by the empty product, which is 1.

The inverse axiom is nearly redundant as well. Since \(P(h_0, h_0) :\cong P(h_0, h_1) \cdot P(h_1, h_0)\), the terms in the constraint look like they must be inverses. But without stating the axiom explicitly, there could be a case where \(P(h_0, h_1)\) is some non-wildcard magnitude like 2 but \(P(h_1, h_0)\) is \(\ast\). This shouldn't be allowed because \(\ast\) represents a lack of knowledge about a value, and we consider \(P(h_1, h_0)\) and \(P(h_1, h_0)\) to be the same piece of information but in reverse.

\subsection{Examples}

Now that the definition of relative probability is squared away, we can construct a library of examples for common RPFs that will serve as building blocks to describing common situations.

\begin{definition}
\label{def:uniform_rpf}
The \textit{uniform} RPF on any \(\Omega\) considers each outcome equally likely. In other words, \(P(h_1, h_2) = 1\) for every pair of outcomes.
\end{definition}

\begin{definition}
\label{def:indeterminate_rpf}
The \textit{indeterminate} RPF has \(P(h_1, h_2) = \ast\) for every pair of outcomes.
\end{definition}

\begin{definition}
A \textit{certain} RPF contains a single outcome that has infinite probability relative to all other outcomes. Let \(h_C\) be the certain outcome with \(h_C \neq h\). Then \(P(h_C, h) = \infty\). The relative probability of the other \(K-1\) outcomes could be anything.
\end{definition}

\begin{definition}
\label{def:empty_rpf}
The \textit{empty} RPF has no outcomes because \(K = 0\), and therefore the function \(P\) has no valid inputs.
\end{definition}

It is surprising that there is still an RPF with \(\Omega = \varnothing\). This is not the case for absolute distributions where such a function does not exist (because with no outcomes, they cannot sum to 1).

\begin{definition}
\label{def:unit_rpf}
The \textit{unit} RPF has a single outcome where \(K = 1\) and \(\Omega = h\). There is only one such RPF where \(P(h, h) = 1\).
\end{definition}

The unit RPF is both uniform and certain. This corresponds to the absolute case where the probability of the single outcome must be 1.

\begin{definition}
\label{def:finite_geometric_rpf}
Let \(P\) be an RPF with K outcomes labeled \((h_0, h_1, ..., h_{K-1})\). \(P\) is a \textit{finite geometric} RPF with ratio \(r\) if the relative probabilities of each outcome with its neighbor is always \(r\). In other words, for all \(i \in (0, 1, ..., K-2)\),
\[P(h_{i+1}, h_i) = r\]
When \(r\) is 0 or \(\infty\), we can call this the \textit{limit finite geometric} RPF.
\end{definition}

Finally, to include an example that is both common and has powerful applications, the relative version of the binomial distribution can be defined as follows:

\begin{definition}
\label{def:binomial_rpf}
A \textit{binomial distribution} has a sample size \(n\), and a probability of success \(p\). The RPF has outcome space \(\Omega = \{0, 1, 2, ..., n\}\) and thus \(K = n + 1\). It is given as follows:
\[P(h_1, h_2) = \frac{h_2!(n-h_2)!}{h_1!(n-h_1)!}\left(\frac{p}{1-p}\right)^{h_1 - h_2}\]
\end{definition}

\section{Concepts for Relative Probability Functions}

We defined the relative probability function in section \ref{section:new_relative_prob} with the fundamental axioms and have constructed some examples. Because new situations arise that do not occur in the Kolmogorov model, we also need to define some new vocabulary.

Figure \ref{fig:flow_chart} gives us a roadmap of these new concepts and their relationship to each other.

\begin{figure}
\begin{tikzpicture}    
[node distance = 1.3in] 

\node (node0) [roundbox] {Relative Probability Function \\ Definition \ref{def:fundamental_laws}} ;

\node (node1) [circlebox, below of=node0] {Anchored RPF\\ Definition \ref{def:anchored_rpf}};

\node (lnode1) [recbox, fill=shaded_green,below left=0.5 and 2.2 of node0] {Empty RPF\\ Definition \ref{def:empty_rpf}};

\node (lnode2) [recbox, fill=shaded_green,below left=2 and 0.5 of node0] {Indeterminate RPF\\ Definition \ref{def:indeterminate_rpf}};

\node (node2) [roundbox, below right= and 2 of node1] {Non-empty\\ Totally Comparable RPF\\ Definition \ref{def:totally_comparable}};

\node (bnode2) [recbox, fill=shaded_green, below right=1 and -3.5 of node2] {Finite Limit \\Geometric Distribution,\\ \(K>2\)\\ Definition \ref{def:finite_geometric_rpf}};

\node (node3) [roundbox, below left= and 2 of node1] {Absolute Probability\\ Function\\ Definition \ref{def:categorical_abs}};

\node (node4) [circlebox, below right=and 2.8 of node3] {Union};

\node (node5) [box1, below left=and -0.5 of node4,yshift=-0.38in] {Facet Probability\\ RPF\\(single impossible event)\\};

\node (node6) [box1, right=-0.07 of node5] {Totally Mutually Possible\\ RPF\\ Definition \ref{def:totally_mutually_possible}\\};

\node (rnode4) [recbox, fill=shaded_green, below of=node6,xshift=-1.0in,yshift=-0.0in] {(Non-Empty)\\ Uniform\\ Definition \ref{def:uniform_rpf}};

\node (rnode3) [recbox, fill=shaded_green, below of=node6,yshift=-0.1in] {Unit\\ Distribution\\ Definition \ref{def:unit_rpf}};

\node (rnode2) [recbox, fill=shaded_green, below of=node6,xshift=1.1in,yshift=-0.0in] {Finite Geometric\\Distribution\\ Definition \ref{def:finite_geometric_rpf}};

\node (rnode1) [recbox, fill=shaded_green, right=1 of node6] {Binomial\\Distribution\\ Definition \ref{def:binomial_rpf}};

\node (lnode3) [recbox, fill=shaded_green, below of=node5,yshift=-0.0in,xshift=-0.5in] {Limit Finite\\Geometric\\ with \(K=2\)};

\path[papLine] (rnode1) -- (node6);

\path[papLine] (rnode2) --+ (node6);

\path[papLine] (rnode3) -- (node6);

\path[papLine] (rnode4) --+ (node6);

\path[papLine] (lnode1) --+ (node0);

\path[papLine] (lnode2) --+ (node0);

\path[papLine] (node1) -- (node0);

\path[papLine] (node2) -- (node1) node[pos=0.4, above, sloped]{Lemma \ref{lemma:totally_comp_anchored}};

\path[papLine] (node3) -- (node1) node[pos=0.5, above, sloped]{Theorem \ref{thm:absolute_anchored}};

\path[papLine] (node4) -- (node3);

\path[papLine] (node4) -- (node2);

\draw[-stealth,transform canvas={xshift=1mm},line width=0.5mm] (lnode3) -- +(0,0.86in);

\draw[-stealth,transform canvas={xshift=24mm},line width=0.5mm] (node5) --+(0,1.22in) node[left,pos=0.5]{Theorem \ref{thm:abs_totally_comparable}};

\path[papLine,dashed] (node1) |- (node3) node[above,pos=0.75]{Matched By} node[below,pos=0.75]{Theorem \ref{thm:absolute_prob_formula}};

\path[papLine] ([xshift=0.6in]bnode2) -- (node2);
\end{tikzpicture}
\caption{This is a roadmap for all of the sub-types of relative probability functions and their relationship to one another. }
\label{fig:flow_chart}
\end{figure}
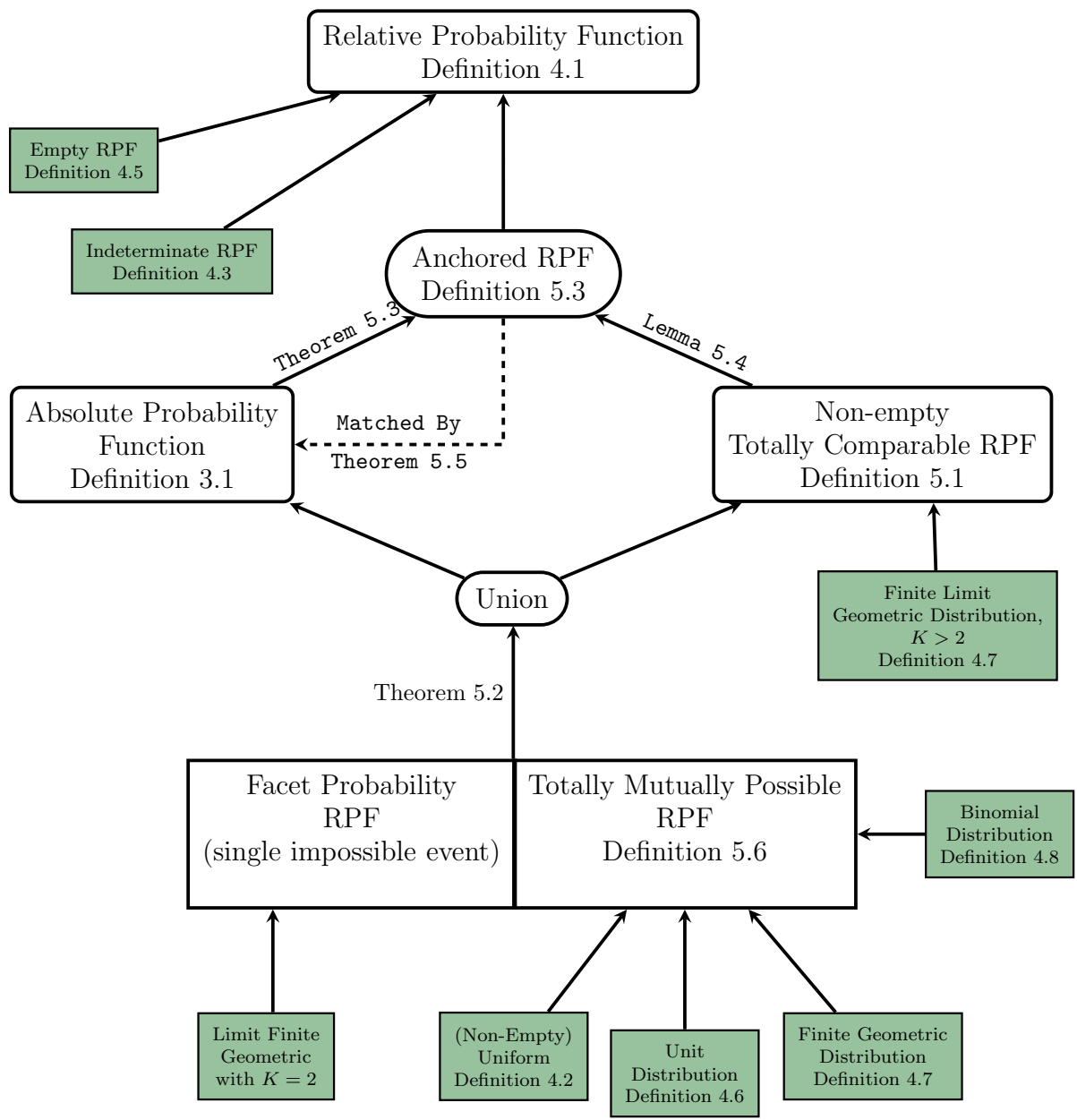

\subsection{Absolute Probability as a Special Case}

Fortunately, the absolute probability function is a special case of an RPF (through definition \ref{def:ratio}), defined by \(P(h_1, h_2) = \frac{P(h_1)}{P(h_2)}\) with the exception that if \(P(h) = 0\), then \(P(h, h) = 1\) instead of \(\ast\) in order to satisfy the identity axiom.

\begin{theorem}
A relative probability function derived from an absolute probability function satisfies the fundamental axioms.
\end{theorem}

\begin{proof}
For the identity axiom, we have \(P(h, h) = 1\) built into the definition.

For the inverse axiom, we can see that for outcomes \(h_1\) and \(h_2\), \(P(h_1, h_2) = \frac{P(h_1)}{P(h_2)} = \left(\frac{P(h_2)}{P(h_1)}\right)^{-1} = P(h_2, h_1)^{-1}\).

For composition, we must show that for all outcomes \(h_1, h_2, h_3\), \(P(h_1, h_3) :\cong P(h_1, h_2) \cdot P(h_2, h_3)\)

Start with the case that \(P(h_2)\neq 0\). Then \(P(h_1, h_2) \cdot P(h_2, h_3) = \frac{P(h_1)}{P(h_2)}\frac{P(h_2)}{P(h_3)} = \frac{P(h_1)}{P(h_3)} = P(h_1, h_3)\). When \(P(h_2) = 0\), \(P(h_1, h_2) \cdot P(h_2, h_3) = \frac{P(h_1)}{P(h_2)}\frac{P(h_2)}{P(h_3)} = \ast\). Because \(\ast\) matches everything, then the matching statement holds. Because it holds in both when \(P(h_2) = 0\) and \(P(h_2) \neq 0\), the theorem holds true always.
\end{proof}

\subsection{Comparability, Possibility, and Anchors}

\begin{definition}
\label{def:totally_comparable}
A relative probability function is \textit{totally comparable} if every pair of outcomes are comparable.
\end{definition}

\begin{theorem}
\label{thm:abs_totally_comparable}
An absolute probability function is totally comparable if and only if \(P(h) = 0\) for at most one outcome.
\end{theorem}

\begin{proof}
Let P be an \textbf{absolute} probability function with distinct outcomes \(h_1\) and \(h_2\). If \(P(h_1) = P(h_2) = 0\), then \(P(h_1, h_2) = \frac{0}{0} = \ast\). If only outcome \(h_1\) is assigned 0, then \(P(h_1, h_1) = 1\), \(P(h_1, h_2) = 0\), and \(P(h_2, h_1) = \infty\). Any other pairing that does not involve \(h_1\) will be the quotient of two positive numbers, and thus not \(\ast\).
\end{proof}

Once it is established that outcomes are comparable to one another, they form groupings. If an outcomes \(h_1\) and \(h_2\) both have probability 0 (or \(\infty\)) with respect to \(h_3\) then \(P(h_1, h_2)\) is not predetermined by composition. This leads us to the concept of possibility and impossibility.

\begin{definition}
Outcome \(h_1\) is \textit{impossible} with respect to \(h_2\) if \(P(h_1, h_2) = 0\). Outcome \(h_1\) is \textit{possible} with respect to \(h_2\) if it is comparable and \(P(h_1, h_2) > 0\).
\end{definition}

Anchor outcomes are those outcomes that have a non-zero absolute probability. The anchoring of a distribution ensures that it is well behaved.

\begin{definition}
\label{def:anchored_rpf}
An \textit{anchor} of an RPF is an outcome that is possible with respect to every other outcome. An RPF that has at least 1 anchor is called an \textit{anchored RPF}.
\end{definition}

\begin{theorem}
\label{thm:absolute_anchored}
All absolute probability distributions are anchored.
\end{theorem}

\begin{proof}
Let P be an absolute probability distribution on \(\Omega\). Because \(\sum_{h \in \Omega} P(h) = 1\), there must be at least one \(h\) such that \(P(h) > 0\).  Therefore, for any comparison outcome \(h'\), \(P(h, h') = \frac{P(h)}{P(h')} > 0\).
\end{proof}

\begin{lemma}
\label{lemma:totally_comp_anchored}
Every non-empty, totally comparable RPF is anchored.
\end{lemma}

\begin{proof}
Let \(P\) be a non-empty and totally comparable RPF. To prove by contradiction, assume that arbitrary outcome \(h\) is not an achor, and therefore there exists another outcome \(h'\) such that \(P(h, h') = 0\).

Create a function \(f: \Omega \rightarrow \Omega\) that finds a potential \(h' = f(h)\) so that \(P(h, f(h)) = 0\).

Let \(f^n\) be the function \(f\) applied n times. Then \(P(h, f^n(h)) = 0\) for all n greater than 0. This is by induction because the case of \(n = 1\) was assumed above, and for the inductive step
\[P(h, f^{n+1}(h)) :\cong P(h, f^n(h)) \cdot P(f^n(h), f(f^n(h)) = 0 \cdot 0 = 0\]

Because \(\Omega\) is finite, repeated applications of \(f\) on \(h\) must eventually return to an outcome that has already been visited. In more rigorous terms, there exists an \(N\) such that \(f^N(h) = f^i(h)\) for some \(i < N\).

This is a contradiction because \(P(f^i(h), f^N(h))\) should equal 0 by the argument above, but 1 by the identity axiom.
\end{proof}

A totally comparable RPF contains the maximum amount of information about the relative probability of its outcomes. Some RPFs have less information but are nevertheless consistent with those that have more. The following definition describes this relationship.

\begin{definition}
Let \(P_1\) and \(P_2\) be relative probability functions. \(P_1\) is matched by \(P_2\) if and only if for all outcomes \(h_1\) and \(h_2\),
\[P_1(h_1, h_2) :\cong P_2(h_1, h_2).\]
\end{definition}

\begin{theorem}
\label{thm:absolute_prob_formula}
Every anchored RPF is matched by an absolute probability function, given by the following equation where \(a\) is an anchor outcome.
\[P(h) = \frac{P(h, a)}{\sum_{h' \in \Omega}P(h', a)}\]
\end{theorem}

\begin{proof}
We need to show that \(P(h)\) is a valid absolute probability function, and that it matches the original RPF.

Because \(a\) is an anchor element, \(P(h', a) < \infty\). Therefore the sum \(\sum_{h' \in \Omega}P(h', a) < \infty\). \(\sum_{h' \in \Omega}P(h', a)\) is also \(>0\), because it includes the term \(P(a, a) = 1\). The numerator \(P(h, a)\) is also \(< \infty\). Therefore, \(P(h) \notin \{\infty, \ast\}\).

We next check that the values of \(P(h)\) sum to 1 as follows:
\[\sum_{h \in \Omega}P(h) = \sum_{h \in \Omega} \frac{P(h, a)}{\sum_{h' \in \Omega}P(h', a)} = \frac{\sum_{h \in \Omega}P(h, a)}{\sum_{h' \in \Omega}P(h', a)} = 1\]

Cancellation of these equal sums is justified because we have shown previously that they cannot be \(0\) or \(\infty\).

Therefore, \(P(h)\) is a valid absolute probability function. We show that the RPF is matched by it though the following calculation:
\begin{equation}
P(h_1, h_2) :\cong P(h_1, a) \cdot P(a, h_2) = \frac{P(h_1, a)}{\sum_{h' \in \Omega}P(h', a)} \div \frac{P(h_2, a)}{\sum_{h' \in \Omega}P(h', a)} = \frac{P(h_1)}{P(h_2)}
\end{equation}
\end{proof}

\subsection{Mutual Possibility}

\begin{definition}
Outcomes \(h_1\) and \(h_2\) are \textit{mutually possible} if they are comparable and \(0 < P(h_1, h_2) < \infty\). In other words, \(h_1\) and \(h_2\) are each possible with respect to the other.
\end{definition}

\begin{theorem}
Mutual possibility is an \textit{equivalence relation}, being reflexive, symmetric and transitive.
\end{theorem}

\begin{proof}
Reflexive: \(P(h, h) = 1\) by the identity axiom.

Symmetric: \(P(h_1, h_2) = P(h_2, h_1)^{-1}\), which means that each can be in \(\{0, \infty, \ast\}\) if and only if the other is as well.

Transitive:  The composition axiom states that \(P(h_1, h_3) :\cong P(h_1, h_2) \cdot P(h_2, h_3)\). If the factors in the constraint (\(P(h_1, h_2)\) and \(P(h_2, h_3)\)) are positive and finite, then their product is also positive and finite.
\end{proof}

\begin{definition}
\label{def:totally_mutually_possible}
An RPF is \textit{totally mutually possible} if all of its outcomes\footnote{This is one of the few definitions that cannot later be upgraded from outcomes to events. The empty event \(e = \{\}\) for example will be impossible with respect to any outcome by theorem \ref{thm:empty_event_impossible}.} are mutually possible or equivalently are all anchors.
\end{definition}

It is helpful to make diagrams of possibility and impossibility through a \textit{directed graph}. In these graphs, each outcome is represented by a point, and an arrow from A to B means that B is possible with respect to A. A bidirectional arrow means that A and B are mutually possible. Totally mutually possible RPFs have a simple diagram where all the outcomes are connected as in figure \ref{fig:mutually_possible_rpf}.

\begin{figure}[h]
\centering
\resizebox{0.2\textwidth}{!}{
\begin{tikzpicture} 
\node(t)[circle,draw,inner sep=4pt,outer sep=1pt,line width=0.2mm,fill=black] at (0,2) {}; 
\node(ml)[circle,draw,inner sep=4pt,outer sep=1pt,line width=0.2mm,fill=black] at (-2,0.5){}; 
\node(mr)[circle,draw,inner sep=4pt,outer sep=1pt,line width=0.2mm,fill=black] at (2,0.5) {}; 
\node(bl)[circle,draw,inner sep=4pt,outer sep=1pt,line width=0.2mm,fill=black] at (-1.3,-1.8){}; 
\node(br)[circle,draw,inner sep=4pt,outer sep=1pt,line width=0.2mm,fill=black] at (1.3,-1.8){}; 

\path[stealth-stealth,draw,line width=0.3mm] (t)--(ml);
\path[stealth-stealth,draw,line width=0.3mm] (t)--(mr);
\path[stealth-stealth,draw,line width=0.3mm] (t)--(bl);
\path[stealth-stealth,draw,line width=0.3mm] (t)--(br);

\path[stealth-stealth,draw,line width=0.3mm] (ml)--(mr);
\path[stealth-stealth,draw,line width=0.3mm] (ml)--(br);
\path[stealth-stealth,draw,line width=0.3mm] (ml)--(bl);

\path[stealth-stealth,draw,line width=0.3mm] (mr)--(br);
\path[stealth-stealth,draw,line width=0.3mm] (mr)--(bl);
\path[stealth-stealth,draw,line width=0.3mm] (bl)--(br);

\end{tikzpicture}
}
\caption{A totally mutually possible RFP has - unsurprisingly - a complete graph of mutual possibility.}
\label{fig:mutually_possible_rpf}
\end{figure}
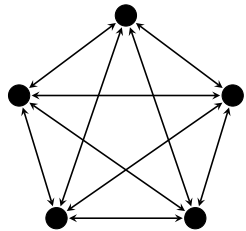

\begin{theorem}
A non-empty totally mutually possible RPF is equal to an absolute probability function.
\end{theorem}

\begin{proof}
If \(P\) is totally mutually possible, then all of its outcomes are anchors. Therefore, we can use theorem \ref{thm:absolute_prob_formula} to find a matching absolute probability function
\[P(h) = \frac{P(h, a)}{\sum_{h' \in \Omega}P(h', a)}\]

Because every element of \(\Omega\) is an anchor, we can let \(a = h\) and get
\[P(h) = \frac{P(h, h)}{\sum_{h' \in \Omega}P(h', h)}=\frac{1}{\sum_{h' \in \Omega}P(h', h)}\]

Theorem \ref{thm:absolute_prob_formula} states that \(P(h_1, h_2) :\cong \frac{P(h_1)}{P(h_2)}\), but since the constraint \(\frac{P(h_1)}{P(h_2)}\) is never \(\ast\), they must be equal.
\end{proof}

\subsection{Possibility Classes}
\label{section:possibility_classes}

In order to analyze the general case of RPFs, we need to consider classes of mutual possibility.

\begin{theorem}
The relationship of being possible is both reflexive and transitive, or in other words a \textit{preorder}.
\end{theorem}

\begin{proof}
It is reflexive because \(P(h, h) = 1\). If \(P(h_1, h_2) > 0\) and \(P(h_2, h_3) > 0\) then their product is also greater than zero, and by composition, equal to \(P(h_1, h_3)\). Thus \(h_1\) is also possible with respect to \(h_3\).
\end{proof}

If we consider the equivalence classes of mutual possibility and their relationship to one another, then we have a \textit{partial order}. Figures \ref{fig:anchored_rpf} and \ref{fig:unanchored_rpf} are both examples of a graph of outcomes grouped by mutually possible equivalence classes. Figure \ref{fig:anchored_rpf} is anchored while figure \ref{fig:unanchored_rpf} is not.

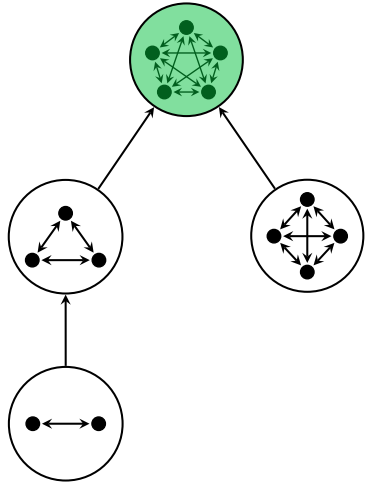
\begin{figure}[h]
\centering
\resizebox{0.3\textwidth}{!}{
\begin{tikzpicture} 
\begin{scope}[scale=0.25]
\node(t)[circle,draw,inner sep=2pt,outer sep=1pt,line width=0.2mm,fill=black] at (0,2) {}; 
\node(ml)[circle,draw,inner sep=2pt,outer sep=1pt,line width=0.2mm,fill=black] at (-2,0.5){}; 
\node(mr)[circle,draw,inner sep=2pt,outer sep=1pt,line width=0.2mm,fill=black] at (2,0.5) {}; 
\node(bl)[circle,draw,inner sep=2pt,outer sep=1pt,line width=0.2mm,fill=black] at (-1.3,-1.8){}; 
\node(br)[circle,draw,inner sep=2pt,outer sep=1pt,line width=0.2mm,fill=black] at (1.3,-1.8){}; 

\path[stealth-stealth,draw,line width=0.2mm] (t)--(ml);
\path[stealth-stealth,draw,line width=0.2mm] (t)--(mr);
\path[stealth-stealth,draw,line width=0.2mm] (t)--(bl);
\path[stealth-stealth,draw,line width=0.2mm] (t)--(br);

\path[stealth-stealth,draw,line width=0.2mm] (ml)--(mr);
\path[stealth-stealth,draw,line width=0.2mm] (ml)--(br);
\path[stealth-stealth,draw,line width=0.2mm] (ml)--(bl);

\path[stealth-stealth,draw,line width=0.2mm] (mr)--(br);
\path[stealth-stealth,draw,line width=0.2mm] (mr)--(bl);
\path[stealth-stealth,draw,line width=0.2mm] (bl)--(br);

\node[circle,fill=darkgreen, fill opacity = 0.5, draw=black,line width=0.3mm, fit=(t) (ml) (mr) (bl)(br), inner sep=-1pt] (r1) {};
\end{scope}

\begin{scope}[yshift=-0.8in,xshift=0.7in]
\node (t)[circle,draw,inner sep=2pt,outer sep=1pt,line width=0.2mm,fill=black] {}; 
\node (b)[circle,draw,inner sep=2pt,outer sep=1pt,line width=0.2mm,fill=black,below=0.8 of t] {}; 
\node(ml)[circle,draw,inner sep=2pt,outer sep=1pt,line width=0.2mm,fill=black,below left =0.35 and 0.3 of t]{}; 
\node(mr)[circle,draw,inner sep=2pt,outer sep=1pt,line width=0.2mm,fill=black,below right=0.35 and 0.3 of t]{}; 

\path[stealth-stealth,draw,line width=0.3mm] (t)--(ml);
\path[stealth-stealth,draw,line width=0.3mm] (t)--(mr);
\path[stealth-stealth,draw,line width=0.3mm] (ml)--(mr);
\path[stealth-stealth,draw,line width=0.3mm] (b)--(mr);
\path[stealth-stealth,draw,line width=0.3mm] (b)--(ml);
\path[stealth-stealth,draw,line width=0.3mm] (b)--(t);

\node[circle, draw=black,line width=0.3mm, fit=(t) (ml) (mr) (b), inner sep=-1.8pt] (r2) {};
\end{scope}

\begin{scope}[yshift=-0.88in,xshift=-0.7in]
\node (t)[circle,draw,inner sep=2pt,outer sep=1pt,line width=0.2mm,fill=black] {}; 
\node(ml)[circle,draw,inner sep=2pt,outer sep=1pt,line width=0.2mm,fill=black,below left =0.5 and 0.3 of t]{}; 
\node(mr)[circle,draw,inner sep=2pt,outer sep=1pt,line width=0.2mm,fill=black,below right=0.5 and 0.3 of t]{}; 

\path[stealth-stealth,draw,line width=0.3mm] (t)--(ml);
\path[stealth-stealth,draw,line width=0.3mm] (t)--(mr);
\path[stealth-stealth,draw,line width=0.3mm] (ml)--(mr);

\node[circle, draw=black,line width=0.3mm, fit=(t) (ml) (mr), inner sep=1pt] (r3) {};
\end{scope}

\begin{scope}[yshift=-2.1in,xshift=-0.89in]
\node(l)[circle,draw,inner sep=2pt,outer sep=1pt,line width=0.2mm,fill=black]{}; 
\node(r)[circle,draw,inner sep=2pt,outer sep=1pt,line width=0.2mm,fill=black,right=0.7 of l]{}; 

\path[stealth-stealth,draw,line width=0.3mm] (l)--(r);

\node[circle, draw=black,line width=0.3mm, fit=(l) (r), inner sep=4.7pt] (r4) {};
\end{scope}

\path[-stealth,draw,line width=0.3mm] (r2) -- (r1);
\path[-stealth,draw,line width=0.3mm] (r3) -- (r1);
\path[-stealth,draw,line width=0.3mm] (r4) -- (r3);

\end{tikzpicture}
}
\caption{A diagram of an anchored RPF with its mutually possible classes. Each point is an outcome, and each circle is mutually possible class. The anchor class (shaded) is the maximal class in the partial order.}
\label{fig:anchored_rpf}
\end{figure}

\begin{figure}[h]
\centering
\resizebox{0.5\textwidth}{!}{
\begin{tikzpicture} 
\begin{scope}[scale=0.25]
\node(t)[circle,draw,inner sep=2pt,outer sep=1pt,line width=0.2mm,fill=black] at (0,2) {}; 
\node(ml)[circle,draw,inner sep=2pt,outer sep=1pt,line width=0.2mm,fill=black] at (-2,0.5){}; 
\node(mr)[circle,draw,inner sep=2pt,outer sep=1pt,line width=0.2mm,fill=black] at (2,0.5) {}; 
\node(bl)[circle,draw,inner sep=2pt,outer sep=1pt,line width=0.2mm,fill=black] at (-1.3,-1.8){}; 
\node(br)[circle,draw,inner sep=2pt,outer sep=1pt,line width=0.2mm,fill=black] at (1.3,-1.8){}; 

\path[stealth-stealth,draw,line width=0.2mm] (t)--(ml);
\path[stealth-stealth,draw,line width=0.2mm] (t)--(mr);
\path[stealth-stealth,draw,line width=0.2mm] (t)--(bl);
\path[stealth-stealth,draw,line width=0.2mm] (t)--(br);

\path[stealth-stealth,draw,line width=0.2mm] (ml)--(mr);
\path[stealth-stealth,draw,line width=0.2mm] (ml)--(br);
\path[stealth-stealth,draw,line width=0.2mm] (ml)--(bl);

\path[stealth-stealth,draw,line width=0.2mm] (mr)--(br);
\path[stealth-stealth,draw,line width=0.2mm] (mr)--(bl);
\path[stealth-stealth,draw,line width=0.2mm] (bl)--(br);

\node[circle, draw=black,line width=0.3mm, fit=(t) (ml) (mr) (bl)(br), inner sep=-1pt] (r1) {};
\end{scope}

\begin{scope}[yshift=-0.8in,xshift=0.7in]
\node (t)[circle,draw,inner sep=2pt,outer sep=1pt,line width=0.2mm,fill=black] {}; 
\node (b)[circle,draw,inner sep=2pt,outer sep=1pt,line width=0.2mm,fill=black,below=0.8 of t] {}; 
\node(ml)[circle,draw,inner sep=2pt,outer sep=1pt,line width=0.2mm,fill=black,below left =0.35 and 0.3 of t]{}; 
\node(mr)[circle,draw,inner sep=2pt,outer sep=1pt,line width=0.2mm,fill=black,below right=0.35 and 0.3 of t]{}; 

\path[stealth-stealth,draw,line width=0.3mm] (t)--(ml);
\path[stealth-stealth,draw,line width=0.3mm] (t)--(mr);
\path[stealth-stealth,draw,line width=0.3mm] (ml)--(mr);
\path[stealth-stealth,draw,line width=0.3mm] (b)--(mr);
\path[stealth-stealth,draw,line width=0.3mm] (b)--(ml);
\path[stealth-stealth,draw,line width=0.3mm] (b)--(t);

\node[circle, draw=black,line width=0.3mm, fit=(t) (ml) (mr) (b), inner sep=-1.8pt] (r2) {};
\end{scope}

\begin{scope}[yshift=-0.88in,xshift=-0.63in]
\node (t)[circle,draw,inner sep=2pt,outer sep=1pt,line width=0.2mm,fill=black] {}; 
\node(ml)[circle,draw,inner sep=2pt,outer sep=1pt,line width=0.2mm,fill=black,below left =0.5 and 0.3 of t]{}; 
\node(mr)[circle,draw,inner sep=2pt,outer sep=1pt,line width=0.2mm,fill=black,below right=0.5 and 0.3 of t]{}; 

\path[stealth-stealth,draw,line width=0.3mm] (t)--(ml);
\path[stealth-stealth,draw,line width=0.3mm] (t)--(mr);
\path[stealth-stealth,draw,line width=0.3mm] (ml)--(mr);

\node[circle, draw=black,line width=0.3mm, fit=(t) (ml) (mr), inner sep=1pt] (r3) {};
\end{scope}

\begin{scope}[yshift=0in,xshift=0.9in]
\node(l)[circle,draw,inner sep=2pt,outer sep=1pt,line width=0.2mm,fill=black]{}; 

\node[circle, draw=black,line width=0.3mm, fit=(l), inner sep=12.8pt] (r6) {};
\end{scope}

\begin{scope}[xshift=1.5in]
\node(l)[circle,draw,inner sep=2pt,outer sep=1pt,line width=0.2mm,fill=black]{}; 
\node(r)[circle,draw,inner sep=2pt,outer sep=1pt,line width=0.2mm,fill=black,right=0.7 of l]{}; 

\path[stealth-stealth,draw,line width=0.3mm] (l)--(r);

\node[circle, draw=black,line width=0.3mm, fit=(l) (r), inner sep=4.7pt] (r4) {};
\end{scope}

\begin{scope}[yshift=-1.01in,xshift=1.7in]
\node(l)[circle,draw,inner sep=2pt,outer sep=1pt,line width=0.2mm,fill=black]{}; 

\node[circle, draw=black,line width=0.3mm, fit=(l), inner sep=12.8pt] (r5) {};
\end{scope}

\path[-stealth,draw,line width=0.3mm] (r2) -- (r1);
\path[-stealth,draw,line width=0.3mm] (r3) -- (r1);
\path[-stealth,draw,line width=0.3mm] (r2) -- (r6);
\path[-stealth,draw,line width=0.3mm] (r5) -- (r4);
\end{tikzpicture}
}
\caption{A diagram for a single RPF that is not anchored. We cannot turn this into an absolute probability function.}
\label{fig:unanchored_rpf} 
\end{figure}
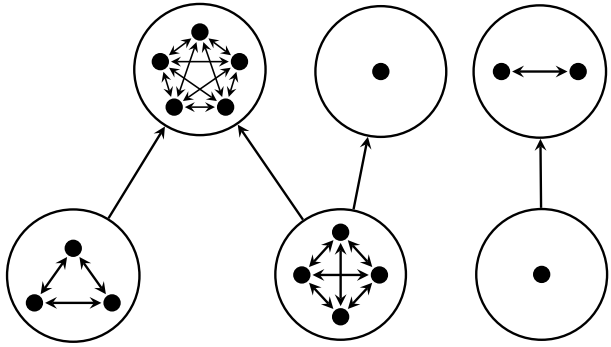

\begin{theorem}
\label{thm:equivalence_class_same}
Consider two distinct equivalence classes of mutually possible outcomes in an RPF and call them \(C_1, C_2 \subseteq \Omega\). Any pair of outcomes \(h_1 \in C_1\) and \(h_2 \in C_2\) will have the same relative probability \(P(h_1, h_2)\) as any other pair of outcomes in \(C_1 \times C_2\). 
\end{theorem}

\begin{proof}
\(P(h_1, h_2)\) must be in \(\{0, \infty, \ast\}\) because otherwise \(h_1\) and \(h_2\) would be mutually possible and in the same equivalence class. Let \(h'_1 \in C_1\) and \(h'_2 \in C_2\) be alternative representatives from equivalences classes \(C_1\) and \(C_2\). Then \(0 < P(h'_1, h_1) < \infty\) and \(0 < P(h_2, h'_2) < \infty\) due to the definition of mutual comparability. Thus with composition we get
\[P(h'_1, h'_2) :\cong P(h'_1, h_1) \cdot P(h_1, h_2) \cdot P(h_2, h'_2) = P(h_1, h_2)\]
\end{proof}

Finally, we look at totally comparable RPFs, where mutually possible components form a \textit{total order} (see figure \ref{fig:totally_comparable_rpf}). We've established in the proof of theorem \ref{thm:equivalence_class_same} that the relative probability of outcomes from one equivalence class to another are constant and either 0, \(\infty\) or \(\ast\). If the RPF is totally comparable, then it can only be 0 or \(\infty\), providing a transitive, binary relation and thus a total order.

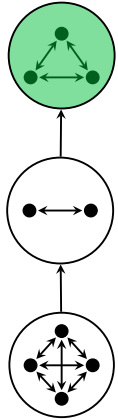
\begin{figure}[h]
\centering
\resizebox{0.1\textwidth}{!}{
\begin{tikzpicture} 
\begin{scope}
\node (t)[circle,draw,inner sep=2pt,outer sep=1pt,line width=0.2mm,fill=black] {}; 
\node(ml)[circle,draw,inner sep=2pt,outer sep=1pt,line width=0.2mm,fill=black,below left =0.5 and 0.3 of t]{}; 
\node(mr)[circle,draw,inner sep=2pt,outer sep=1pt,line width=0.2mm,fill=black,below right=0.5 and 0.3 of t]{}; 

\path[stealth-stealth,draw,line width=0.3mm] (t)--(ml);
\path[stealth-stealth,draw,line width=0.3mm] (t)--(mr);
\path[stealth-stealth,draw,line width=0.3mm] (ml)--(mr);

\node[circle, fill=darkgreen, fill opacity = 0.5, draw=black,line width=0.3mm, fit=(t) (ml) (mr), inner sep=1pt] (r1) {};
\end{scope}

\begin{scope}[yshift=-1.1in,xshift=-0.2in]
\node(l)[circle,draw,inner sep=2pt,outer sep=1pt,line width=0.2mm,fill=black]{}; 
\node(r)[circle,draw,inner sep=2pt,outer sep=1pt,line width=0.2mm,fill=black,right=0.7 of l]{}; 

\path[stealth-stealth,draw,line width=0.3mm] (l)--(r);

\node[circle, draw=black,line width=0.3mm, fit=(l) (r), inner sep=4.7pt] (r2) {};
\end{scope}

\begin{scope}[yshift=-1.85in]
\node (t)[circle,draw,inner sep=2pt,outer sep=1pt,line width=0.2mm,fill=black] {}; 
\node (b)[circle,draw,inner sep=2pt,outer sep=1pt,line width=0.2mm,fill=black,below=0.8 of t] {}; 
\node(ml)[circle,draw,inner sep=2pt,outer sep=1pt,line width=0.2mm,fill=black,below left =0.35 and 0.3 of t]{}; 
\node(mr)[circle,draw,inner sep=2pt,outer sep=1pt,line width=0.2mm,fill=black,below right=0.35 and 0.3 of t]{}; 

\path[stealth-stealth,draw,line width=0.3mm] (t)--(ml);
\path[stealth-stealth,draw,line width=0.3mm] (t)--(mr);
\path[stealth-stealth,draw,line width=0.3mm] (ml)--(mr);
\path[stealth-stealth,draw,line width=0.3mm] (b)--(mr);
\path[stealth-stealth,draw,line width=0.3mm] (b)--(ml);
\path[stealth-stealth,draw,line width=0.3mm] (b)--(t);

\node[circle, draw=black,line width=0.3mm, fit=(t) (ml) (mr) (b), inner sep=-1.8pt] (r3) {};
\end{scope}

\path[-stealth,draw,line width=0.3mm] (r2) -- (r1);
\path[-stealth,draw,line width=0.3mm] (r3) -- (r2);
\end{tikzpicture}
}
\caption{A diagram of a totally comparable RPF that is not totally mutually possible. The mutually possible components form a total order, with the \textit{anchored component} on top and the graph forming a line.}
\label{fig:totally_comparable_rpf}
\end{figure}

\section{From Outcomes to Events}
\label{section:outcomes_to_events}

Our next task is to upgrade \(P\) to operate on the event level. This is more difficult than it seems. For example, we may wish to declare that the probability of event \(e_1\) with respect to \(e_2\) is going to be additive on \(e_1\) as follows:
\begin{equation}
\label{eq:incorrect_additive_event_def}
P(e_1, e_2) = \sum_{h_1 \in e_1}P(h_1, e_2)
\end{equation}

Equation \ref{eq:incorrect_additive_event_def} looks uncontroversial, but it actually contradicts the fundamental axioms! If we let \(e_1 = \varnothing\), then we have an empty sum on the right hand side of the equation, and we get \(P(\varnothing, e_2) = 0\). Likewise, if we allow \(e_2\) to be empty, we get \(P(e_1, \varnothing) = P(\varnothing, e_1)^{-1}=0^{-1}=\infty\). Both of these statements make sense until they collide in \(P(\varnothing, \varnothing) = 0 = \infty\), and what's worse is that this is also equal 1 under the identity axiom!

Another problem arises when an event is \textit{internally incomparable}, meaning that it contains outcomes that are incomparable with each other. Perhaps there are interesting things we can say about such events, but here we will constrain ourselves to totally comparable RPFs in order to avoid such questions.

\begin{definition}
\label{def:event_comparison}
Let \(P\) be a totally comparable RPF. \(P\) measures the probability of two events relative to each other using the following rules:

\begin{enumerate}[(i)]
  \item \label{event_def_1} \(P(e_1, e_2)\) obeys the fundamental axioms of relative probability.
  \item \label{event_def_2} \(P(e_1, e_2)\) sums over any reference outcome \(r\) as follows:
    \begin{equation}
      \label{eq:event_def_ratio_match}
      P(e_1, e_2) :\cong \frac{\sum_{h_1 \in e_1} P(h_1, r)}{\sum_{h_2 \in e_2} P(h_2, r)}.
    \end{equation}
\end{enumerate}
\end{definition}

Because we no longer have access to absolute probability, the best we can do is measure it relative to a \textit{reference outcome} \(r\). This ratio might be indeterminate, so we use the matching relation instead of equality. Fortunately, we can show that there exists at least one reference outcome that will constrain \(P(e_1, e_2)\) in statement \ref{eq:event_def_ratio_match} so long as \(e_1\) and \(e_2\) are not both empty.

\begin{proof} Lemma \ref{lemma:totally_comp_anchored} states that all totally comparable RPFs have anchor outcomes, and therefore (by the same argument) every event must contain \textbf{internal anchors}, defined as an anchor for the RPF that is restricted to outcomes in that event. Choose an \textit{internal anchor} \(a\) from one of the events, say \(e_1\). Then the sum \(\sum_{h_1 \in e_1} P(h_1, a)\) will be non-infinite\footnote{For anchor a, \(P(a, h) > 0\) but this also means that \(P(h, a) < \infty\).}, and non-zero because \(P(a, a) = 1\) is a term in the sum. Therefore, the constraint as a whole cannot be indeterminate.
\end{proof}

If both events are empty then they cannot have internal anchors, but by the identity axiom \(P(\varnothing, \varnothing) = 1\).

These requirements again seem reasonable, but how can we know for sure that they provide a complete and consistent definition of \(P: \mathcal{F} \times \mathcal{F} \rightarrow \mathbb{M}\)? The following must be shown:

\begin{enumerate}[(i)]
  \item \label{event_def_proof_1} If two distinct reference outcomes for \(r\) in statement \ref{eq:event_def_ratio_match} yield constraints on \(P\), then they must be equal.
  \item \label{event_def_proof_2} The constraint in statement \ref{eq:event_def_ratio_match} does not violate the fundamental axioms.
\end{enumerate}

\begin{proof}
For \ref{event_def_proof_1}:

Let \(r_1\) and \(r_2\) be distinct reference outcomes that constrain \(P(e_1, e_2)\). Then we want to check that

\begin{equation}
\label{eq:relative_event_unique}
\frac{\sum_{h_1 \in e_1} P(h_1, r_1)}{\sum_{h_2 \in e_2} P(h_2, r_1)} = \frac{\sum_{h_1 \in e_1} P(h_1, r_2)}{\sum_{h_2 \in e_2} P(h_2, r_2)}
\end{equation}

Neither expression is a wildcard, and none of the individual terms are either. The key to this argument is in the value of \(P(r_1, r_2)\).

Assume \(P(r_1, r_2) = 0\). 

If \(\sum_{h_1 \in e_1} P(h_1, r_1)\) is finite, then \(\sum_{h_1 \in e_1} P(h_1, r_2)\) must be matched by \(\sum_{h_1 \in e_1} P(h_1, r_1) \cdot P(r_1, r_2)\) which is 0. By the same argument, \(\sum_{h_2 \in e_2} P(h_2, r_2)\) is also 0. Since they can't both be zero, one of the sums involving \(r_1\) on the right hand side is infinite. Therefore \(P(e_1, e_2)\) is either \(\infty\) or 0. Let's say it is \(P(e_1, e_2) = 0\). Then \(\sum_{h_2 \in e_2} P(h_2, r_1) = \infty\) and by the argument above \(\sum_{h_1 \in e_1} P(h_1, r_2) = 0\). Because the right hand side is not \(\ast\) - it must resolve to zero as well.

By analogous arguments, equation \ref{eq:relative_event_unique} must also hold when \(P(e_1, e_2) = \infty\) or \(P(r_1, r_2) = \infty\) or both.

Next assume that \(P(r_1, r_2) \notin \{0, \infty\} \). Multiply the left hand side of equation \ref{eq:relative_event_unique} by \(1 = \frac{P(r_1, r_2)}{P(r_1, r_2)}\) and distribute to get:

\[\frac{\sum_{h_1 \in e_1} P(h_1, r_1) \cdot P(r_1, r_2)}{\sum_{h_2 \in e_2} P(h_2, r_1) \cdot P(r_1, r_2)} = \frac{\sum_{h_1 \in e_1} P(h_1, r_2)}{\sum_{h_2 \in e_2} P(h_2, r_2)}\]

For \ref{event_def_proof_2}:

The identity, inverse, and composition axioms follow from the fact that statement \ref{eq:event_def_ratio_match} is a ratio with identical expressions for \(e_1\) in the numerator and \(e_2\) in the denominator. Therefore, if it resolves it is just a ratio of positive numbers, which follow the fundamental axioms.
\end{proof}

\begin{theorem}
If events \(e_1\) and \(e_2\) are not both empty, the following formula calculates the relative probability of events:
\[P(e_1, e_2) = \sum_{h_1 \in e_1} \frac{1}{\sum_{h_2 \in e_2} P(h_2, h_1)}.\]
\end{theorem}

\begin{proof}
Find a suitable reference outcome $r$ and multiply by \(1 = \frac{P(h_1, r)}{P(h_1, r)}\).
\[\sum_{h_1 \in e_1} \frac{1}{\sum_{h_2 \in e_2} p(h_2, h_1)} :\cong \sum_{h_1 \in e_1} \frac{P(h_1, r)}{\sum_{h_2 \in e_2} P(h_2, h_1) P(h_1, r)} = \frac{\sum_{h_1 \in e_1} P(h_1, r)}{\sum_{h_2 \in e_2} P(h_2, r)}\]

Since both \(P(e_1, e_2)\) and \(\sum_{h_1 \in e_1} \frac{1}{\sum_{h_2 \in e_2} p(h_2, h_1)}\) match  \(\frac{\sum_{h_1 \in e_1} P(h_1, r)}{\sum_{h_2 \in e_2} P(h_2, r)}\) which is not \(\ast\) for appropriate reference r, they must be equal.
\end{proof}

We then derive the absolute probability function as
\[P(e) = P(e, \Omega) = \sum_{h \in e} \frac{1}{\sum_{h' \in \Omega}p(h', h)}\]

\begin{theorem}
\label{thm:empty_event_impossible}
The empty event \(\varnothing\) has probability 0 relative to any non-empty event.
\end{theorem}

\begin{proof}
Let \(e\) be a non-empty event, and let \(h\) be an outcome in \(e\).

\[P(\varnothing, e) :\cong \frac{\sum_{h_1 \in \varnothing} P(h_1, h)}{\sum_{h_2 \in e} P(h_2, h)} = \frac{0}{\sum_{h_2 \in e} P(h_2, h)}\]

The sum \(\sum_{h_2 \in e} P(h_2, h)\) cannot itself be zero because \(P(h, h)\) is one of its terms. Therefore, \(P(\varnothing, e) = 0\)
\end{proof}

\section{Composing Relative Probability Functions}

Let \(P_0, P_1, ..., P_{K-1}\) be relative probability functions on outcome spaces \(\Omega_0, \Omega_1, ..., \Omega_{K-1}\) respectively so that each \(P_k\) is a function of type \(\Omega_k \times \Omega_k \rightarrow \mathbb{M}^{\ast}\).

We can combine all of these relative probability functions together with a top level probability function \(P_\top\) (pronounced ``P-Top'') with outcome space \(\Omega_\top = \{\Omega_0, \Omega_1, ... \Omega_{K- 1}\}\). The outcome space is hierarchical as shown in figure \ref{fig:compose_rpfs}.

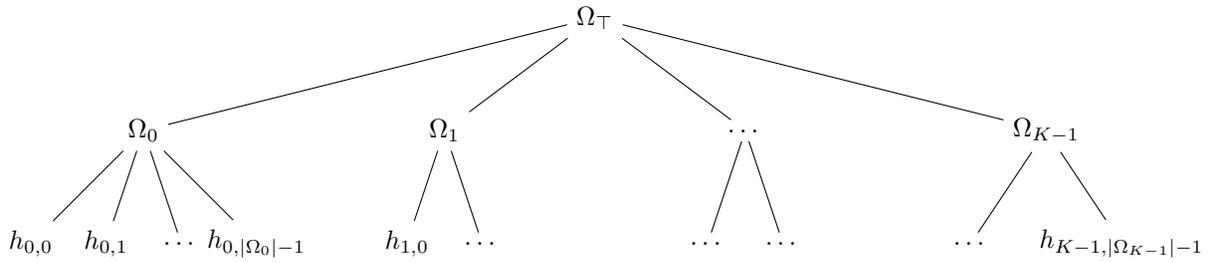
\begin{figure}[h]
\begin{tikzpicture}
\node {\(\Omega_\top\)} [sibling distance = 4cm]
  child {node {\(\Omega_0\)}  [sibling distance = 1cm]
    child {node {\(h_{0, 0}\)}}
    child {node {\(h_{0, 1}\)}}
    child {node {\dots}}
    child {node {\(h_{0, |\Omega_0| - 1}\)}}
  }
  child {node {\(\Omega_1\)}  [sibling distance = 1cm]
    child {node {\(h_{1, 0}\)}}
    child {node {\dots}}  
  }
  child {node {\dots}  [sibling distance = 1cm]
    child {node {\dots}}
    child {node {\dots}}
  }
  child {node {\(\Omega_{K-1}\)}   [sibling distance = 2cm]
    child {node {\dots}}
    child {node {\(h_{K-1, |\Omega_{K-1}| - 1}\)}}
  };
\end{tikzpicture}
\caption{A tree diagram for a set of RPFs being composed by a top-level RPF.}
\label{fig:compose_rpfs}
\end{figure}

Now let \(\Omega\) be the set of all outcomes \(\Omega_0 \cup \Omega_1 \cup \dots \Omega_{K-1}\). We can create a new relative probability function \(P\) acting on outcome space \(\Omega\) with the following rules:

\begin{itemize}
\item If the two outcomes fall under the same component, then their relative probabilities do not change:

\begin{equation}
\label{rpf_composition_same_branch}
P(h_{k, i}, h_{k, j}) = P_k(h_{k, i}, h_{k, j})
\end{equation}

\item If the two outcomes fall under different components, then their relative probabilities are given as follows.

\begin{equation}
\label{eq:rpf_composition_different_branch}
P(h_{k_1, i}, h_{k_2, j}) = P_{k_1}(h_{k_1, i}, \Omega_{k_1}) \cdot  P_{\top}(\Omega_{k_1}, \Omega_{k_2}) \cdot P_{k_2}(\Omega_{k_2}, h_{k_2, j})
\end{equation}
\end{itemize}

Note the use of composition to traverse up and down the tree. One could of course imagine this tree being many levels, and having a different height for each branch.

\begin{theorem}
Composed function \(P\) as defined above respects the fundamental axioms.
\end{theorem}

\begin{proof}
The axioms are already obeyed by each component internally, but they also need to hold when the two inputs to \(P\) are in \textbf{different} components. An outcome is always in the same component as itself, so the identity axiom immediately follows from equation \ref{rpf_composition_same_branch}.

The inverse law can be proven by calculation.
\begin{equation}
\begin{aligned}
P(h_{k_1, i}, h_{k_2, j})^{-1} &= (P_{k_1}(h_{k_1, i}, \Omega_{k_1}) \cdot  P_{\top}(\Omega_{k_1}, \Omega_{k_2}) \cdot P_{k_2}(\Omega_{k_2}, h_{k_2, j}))^{-1} \\
& = P_{k_1}(h_{k_1, i}, \Omega_{k_1})^{-1} \cdot  P_{\top}(\Omega_{k_1}, \Omega_{k_2})^{-1} \cdot P_{k_2}(\Omega_{k_2}, h_{k_2, j})^{-1} \\
& = P_{k_1}(\Omega_{k_1}, h_{k_1, i}) \cdot  P_{\top}(\Omega_{k_2}, \Omega_{k_1}) \cdot P_{k_2}(h_{k_2, j}, \Omega_{k_2}) \\
& = P_{k_2}(h_{k_2, j}, \Omega_{k_2})\cdot  P_{\top}(\Omega_{k_2}, \Omega_{k_1}) \cdot P_{k_1}(\Omega_{k_1}, h_{k_1, i}) \\
& = P(h_{k_2, j}, h_{k_1, i})
\end{aligned}
\end{equation}

Composition can be shown similarly - now naming the 3 separate indices in components \(k_1, k_2, k_3\) as \(i_1, i_2, i_3\) respectively.
\begin{equation}
\begin{aligned}
& P(h_{k_1, i_1}, h_{k_2, i_2}) \cdot P(h_{k_2, i_2}, h_{k_3, i_3}) \\
& = P_{k_1}(h_{k_1, i_1}, \Omega_{k_1}) \cdot  P_{\top}(\Omega_{k_1}, \Omega_{k_2}) \cdot \textcolor{red}{P_{k_2}(\Omega_{k_2}, h_{k_2, i_2}) \cdot  P_{k_2}(h_{k_2, i_2}, \Omega_{k_2})} \cdot  P_{\top}(\Omega_{k_2}, \Omega_{k_3}) \cdot P_{k_3}(\Omega_{k_3}, h_{k_3, i_3}) \\
& :\cong P_{k_1}(h_{k_1, i_1}, \Omega_{k_1}) \cdot  \textcolor{red}{P_{\top}(\Omega_{k_1}, \Omega_{k_2}) \cdot  P_{\top}(\Omega_{k_2}, \Omega_{k_3})} \cdot P_{k_3}(\Omega_{k_3}, h_{k_3, i_3}) \\
& :\cong P_{k_1}(h_{k_1, i_1}, \Omega_{k_1}) \cdot  P_{\top}(\Omega_{k_1}, \Omega_{k_3}) \cdot P_{k_3}(\Omega_{k_3}, h_{k_3, i_3}) \\
& :\cong P_{k_1}(h_{k_1, i_1}, h_{k_3, i_3})
\end{aligned}
\end{equation}
\end{proof}

\begin{theorem}
\(P\) is totally comparable if and only if the following are true:

\begin{enumerate}[(i)]
\item \label{comp_recs_1} \(P_{\top}\) and \(P_k\) for all \(k \in \{0, 1, ..., K - 1\}\) are totally comparable.
\item \label{comp_recs_2} All components except at most one are totally mutually possible.
\item \label{comp_recs_3} If there is a component that is not totally mutually possible, then every element of \(P_{\top}\) is possible with respect to that component.
\end{enumerate}
\end{theorem}

\begin{proof}
First assume the conditions to show that the total comparability of \(P\) follows.

If all the components are totally comparable, then any two outcomes in the same component are always going to be comparable in the overall RPF. We only need to prove that outcomes in \textbf{different} components are comparable. Starting with equation \ref{eq:rpf_composition_different_branch},

\begin{equation}
P(h_{k_1, i}, h_{k_2, j}) = P_{k_1}(h_{k_1, i}, \Omega_{k_1}) \cdot  P_{\top}(\Omega_{k_1}, \Omega_{k_2}) \cdot P_{k_2}(\Omega_{k_2}, h_{k_2, j})
\end{equation}

The only way that \(P(h_{k_1, i}, h_{k_2, j}) = \ast\) is if both \(0\) and \(\infty\) are factors of the constraint \(P_{k_1}(h_{k_1, i}, \Omega_{k_1}) \cdot  P_{\top}(\Omega_{k_1}, \Omega_{k_2}) \cdot P_{k_2}(\Omega_{k_2}, h_{k_2, j})\)

Because there is at most one component that is not totally mutually possible by \ref{comp_recs_2}, we can say that either \(P_{k_1}(h_{k_1, i}, \Omega_{k_1}) = 0\) or \(P_{k_2}(h_{k_2, j}, \Omega_{k_2}) = 0\), or possibly neither, but not both.

Using equation \ref{eq:event_def_ratio_match} with the factor \(P_{k_1}(h_{k_1, i}, \Omega_{k_1})\) and \(k_1\) itself as the reference outcome, we get
\[
P_{k_1}(h_{k_1, i}, \Omega_{k_1}) :\cong \frac{\sum_{h_1 \in \{k_1\}} P(h_1, k_1)}{\sum_{h \in \Omega_{k_1}} P(h_2, k_1)} = \frac{1}{\sum_{h \in \Omega_{k_1}} P(h_2, k_1)}.
\]

The sum in the denominator cannot be zero since \(P(k_1, k_1) = 1\) will be one of its terms. Therefore the factors \(P_{k_1}(h_{k_1, i}, \Omega_{k_1})\) and \(P_{k_2}(h_{k_2, j}, \Omega_{k_2})\) also must be finite.

If the term \(P_{k_1}(h_{k_1, i}, \Omega_{k_1}) = 0\), then the only way the entire right hand side can be \(\ast\) is if \(P_{\top}(\Omega_{k_1}, \Omega_{k_2}) = \infty\). But this can't be true because by \ref{comp_recs_3} we assumed that \(\Omega_{k_2}\) is possible with respect to \(\Omega_{k_1}\), the sole component with impossible outcomes. An analogous argument can be made if \(P_{k_2}(h_{k_2, j}, \Omega_{k_2}) = 0\).

Therefore, the right hand side of the equation is not \(\ast\) and \(P\) is totally comparable.

In the opposite direction, we show that if any of the conditions are broken, then \(P\) is not totally comparable. Breaking condition \ref{comp_recs_1} would introduce an explicit \(\ast\) into equation \ref{eq:rpf_composition_different_branch}. If there are multiple components with impossible outcomes breaking condition \ref{comp_recs_2}, then it would introduce a \(0\) into the first term of equation \ref{eq:rpf_composition_different_branch} and an \(\infty\) into the third term, yielding \(\ast\).

And finally, if only condition \ref{comp_recs_3} is broken, it would (with appropriate choices of \(h_1\) and \(h_2\)) introduce a 0 into the first term of equation \ref{eq:rpf_composition_different_branch} and an \(\infty\) into the \textbf{second} term of equation \ref{eq:rpf_composition_different_branch}.

Therefore, if any of these conditions are broken, \(P\) is \textbf{not} totally comparable.
\end{proof}

\section{Bayesian Inference on Relative Distributions}

Bayesian inference is the practice of updating ones beliefs about the state of the world given new data.

A relative probability function represents a belief over the set of potential hypotheses in \(\Omega\), meaning that Bayesian inference can be performed on RPFs. The initial belief is called the \textit{prior} and the updated belief is called the \textit{posterior}. Once a formula for the posterior is is worked out, the practitioner might then want to search the hypothesis space \(\Omega\) for an outcome that is either most likely (\textit{maximum a posteriori}) or very good with respect to the posterior distribution. They might also wish to randomly sample one or more outcomes \(h \in \Omega\) from the posterior weighted according to the relative probabilities.

Almost all of these sampling and search methods rely on an algorithm that starts at one or more initial hypotheses and iterately updates using the relative probability of a current outcome and its nearby outcomes. Examples include hill climbing, simulated annealing, Newton-Raphson, Markov Chain Monte Carlo, and the No U-Turn sampler. The role of these algorithms in supervised machine learning has been previously discussed by Local Maximum Labs in ``Sampling Bias Correction for Supervised Machine Learning\cite{sklar_bias}''. Because relative probability simplifies Bayes rule and is ideal for many sampling and selection algorithms, statisticians and engineers should consider using the RPF framework for these purposes.

Start with the Bayesian inference formula for conditional probability for \(h \in \Omega\) assuming that we receive data \(D\).

\[P(h|D) = \frac{P(D|h) \cdot P(h)}{P(D)} \qquad P(D) = \sum_{h \in \Omega} P(D|h) \cdot P(h)\]

Now we convert to relative probability by looking at two hypotheses \(h_1\) and \(h_2\) and the ratio of their posterior probabilities.

\[\frac{P(h_1|D)}{P(h_2| D)} = \frac{P(D|h_1) \cdot P(h_1)}{P(D)} \div \frac{P(D|h_2) \cdot P(h_2)}{P(D)} = \frac{P(D|h_1) \cdot P(h_1)}{P(D|h_2) \cdot P(h_2)} \]

Notice that each component is now represented by a ratio. By making the appropriate substitutions, we can express this entirely in terms of RPFs.

For the ratio of prior probabilities, substitute the relative prior: \(\frac{P(h_1)}{P(h_2)} \rightarrow P(h_1, h_2) \)

For the ratio of posterior probabilities, substitute the relative posterior: \(\frac{P(h_1|D)}{P(h_2|D)} \rightarrow P(h_1, h_2|D) \)

The liklihood ratio, also known as the \textit{bayes factor}, may not look like a relative probability, but the formula to expand conditional probability suggests that it is:

\[\frac{P(D|h_1)}{P(D|h_2)} = \frac{\frac{P(D \cap h_1)}{P(D)}}{\frac{P(D \cap h_2)}{P(D)}} = \frac{P(D \cap h_1)}{P(D \cap h_2)} \]

Let \(P_D\) represent the likelihood ratio of the different hypotheses. The likelihood ratio \(P_D(h_1, h_2)\) encodes a description of how the different hypotheses rate the likelihood of data.

The substitution for the likelihood ratio is now as follows: \(\frac{P(D|h_1)}{P(D|h_2)} \rightarrow P_D(h_1, h_2) \)

These substitutions create a Bayes rule for relative probability:

\begin{equation}
P(h_1, h_2|D) = P_D(h_1, h_2) P(h_1, h_2)
\end{equation}
 
Bayesian inference is now reduced to a term-by-term multiplication of two different RPFs: \(P_D(h_1, h_2)\) and \(P(h_1, h_2)\). Fortunately, the product of two RPFs also obey the fundamental axioms.

\begin{theorem} 
Let \(P_1\) and \(P_2\) be relative probability functions on \(\Omega\). Define \(P(h_1, h_2) = P_1(h_1, h_2) \cdot P_2(h_1, h_2)\). Then, \(P\) is also an RPF because it obeys the fundamental axioms.
\end{theorem}

\begin{proof}
Identity: \(P(h_1, h_1) = P_1(h_1, h_1) P_2(h_1, h_1)=1 \cdot 1=1\)
 
Inverse: \[P(h_1, h_2) = P_1(h_1, h_2) \cdot P_2(h_1, h_2)=P_1(h_2, h_1)^{-1} \cdot P_2(h_2, h_1)^{-1}=(P_1(h_2, h_1) \cdot P_2(h_2, h_1))^{-1}=P(h_2, h_1)^{-1}\]
 
Composition:
\[P(h_1, h_2)P(h_2, h_3)=P_1(h_1, h_2) P_2(h_1, h_2)P_1(h_2, h_3) P_2(h_2, h_3) :\cong P_1(h_1, h_3) P_2(h_1, h_3)=P(h_1, h_3)\]
\end{proof}

\subsection{Degeneration of Mutually Possible Beliefs}

The \textit{Cromwell rule} in Bayesian inference cautions against setting the probability of outcomes equal to 0 in a prior belief. For relative probability, this corresponds to a preference for totally mutually possible RPFs. If an outcome were considered impossible with respect to another, this would be a permanent belief that cannot be changed through evidence.

\begin{theorem}
If an outcome is impossible with respect to another outcome in the posterior distribution, it will either remain impossible or become incomparable in the posterior. In other words,  if \(P(h_1, h_2)=0\), then \(P(h_1, h_2|D) \in \{0, \ast\}\).
\end{theorem}

\begin{proof}
Assume \(P(h_1, h_2) = 0\). Then \(P(h_1, h_2|D) = P_D(h_1, h_2) P(h_1, h_2) = P_D(h_1, h_2) \cdot 0\). This final term would normally simplify to 0, but will be \(\ast\) if \(P_D(h_1, h_2) \in \{\infty, \ast\}\).
\end{proof}

If a Bayesian setup allows an outcome to be declared impossible with respect to another from one piece of evidence, and the reverse from another piece of evidence, those outcomes will become permanently incomparable. This failure mode will be familiar to those machine learning practitioners who have ever received a model full of \textit{NaNs}.

\begin{theorem}
If two outcomes are incomparable in a prior distribution, they will be incomparable in the posterior distribution. In other words, if \(P(h_1, h_2)=\ast\), then \(P(h_1, h_2|D) = \ast\).
\end{theorem}

\begin{proof}
Assume \(P(h_1, h_2) = \ast\). Then \(P(h_1, h_2|D) = P_D(h_1, h_2) P(h_1, h_2) = P_D(h_1, h_2) \cdot \ast = \ast\)
\end{proof}

\subsection{Example: A Noisy Channel}

Here is an example of how relative probability gives us an interesting way of looking at statistical inference problems.

Suppose we are to receive a message in the outcome space \(\Omega = \{0, 1, ..., K-1\}\). There is a probability of \(p\) that the message goes through correctly. Otherwise, it gets scrambled and we receive a value in \(\Omega\) drawn from the uniform distribution\footnote{We could still have gotten lucky and received the correct value.}. We receive the same message several times for redundancy, and we count \(c_k\) as the number of times the message was received as \(k\).

The indicator function can be used to get the absolute probability of receiving \(h_1\) given that the real message was \(h_2\).

\[P(\text{received}\: h_1 | \text{message}\: h_2) = p[h_1 = h_2] + \frac{1-p}{K}\]
 
We then use this to construct an RPF for the likelihood ratio if we receive a single message, \(k\).
\[P_k(h_1, h_2) = \frac{p[h_1 = k] + \frac{1-p}{K}}{p[h_2 = k] + \frac{1-p}{K}} = \frac{pK[h_1 = k] + 1-p}{pK[h_2 = k] + 1-p}\]

If we receive multiple messages in the count vector \(c\), we get the following likelihood formula:
\[P_c(h_1, h_2) = \prod_{k \in \Omega}\left(\frac{pK[h_1 = k] + 1-p}{pK[h_2 = k] + 1-p}\right)^{c_k}\]

Every factor of the product where \(k \notin \{h_1, h_2\}\) will be 1 and have no effect. Therefore we can remove all those terms from the product.
\begin{equation}
\begin{aligned}
P_c(h_1, h_2) &= \left(\frac{pK[h_1 = h_1] + 1-p}{pK[h_2 = h_1] + 1-p}\right)^{c_{h_1}} \left(\frac{pK[h_1 = h_2] + 1-p}{pK[h_2 = h_2] + 1-p}\right)^{c_{h_2}} \\
& = \left(\frac{pK + 1-p}{1-p}\right)^{c_{h_1}} \left(\frac{1-p}{pK + 1-p}\right)^{c_{h_2}} = \left(1 + \frac{pK}{1-p}\right)^{c_{h_1} - c_{h_2}}
\end{aligned}
\end{equation}

Because the prior is uniform, the posterior is just equal to the likelihood.
\[P(h_1, h_2 | c) = P_c(h_1, h_2) \cdot P(h_1, h_2) = \left(1 + \frac{pK}{1-p}\right)^{c_{h_1} - c_{h_2}} \]

We now have an insight: the relative probability between two hypotheses is exponential on the difference between their counts. Formulating these problems in terms of relative probability often leads to easily interpretable results, even before converting into absolute probability (which may or may not be required). Using a different prior would be as easy as appending an additional term to the formula for \(P(h_1, h_2|c)\).

\subsection{Digital Representation}

Even if an inference problem starts with a mutually possible prior, the posterior could end up as any RPF including all types in figure \ref{fig:flow_chart}. How can we represent an RPF digitally in a data structure that could account for all of these various possibilities?

Mutual possibility classes from section \ref{section:possibility_classes} provide a good framework for organizing an RPF. Every outcome is a member of a mutual possibility class, and has a relative probability within that class.

To that end, we need to maintain a mapping from outcomes to classes, and a partial order of all the mutual possibility classes. The data structure for this purpose could be as simple as listing the outbound edges for each class (illustrated by the graphs in figures \ref{fig:anchored_rpf}, \ref{fig:unanchored_rpf}, and \ref{fig:totally_comparable_rpf}). There needs to be a method to compare two classes, and the comparison will return one of four values: greater, lesser, equal, or incomparable - or in numerical terms \(\{\infty, 0, 1, \ast\}\).

Each comparison might require a traversal of the mutual possibility graph which could get expensive in certain situations. The choice of data structure for partial orders comes with tradeoffs. For small outcome spaces graph traversal will usually be adequate. When Bayesian inference is performed, the mutual possibility classes will sometimes need to be split up - either vertically as some items in a class become impossible with respect to others - or horizontally as they become incomparable. The data structure should be able to account for this as well.

Let \(\Omega\) be the outcome space, and let \(C\) be the set of mutually possible classes on \(\Omega\) with respect to the RPF. We maintain three functions 

\begin{itemize}
\item \(\alpha: \Omega \rightarrow C\) assigns outcomes to possibility classes.
\item \(\ell: \Omega \rightarrow (-\infty, \infty)\) provides a log value for an outcome within its mutual possibility class. This allows us to use the full range of floating point numbers available on our machine.
\item \(Q: C \times C \rightarrow \{\infty, 0, 1, \ast\}\) compares two probability classes.
\end{itemize}

From these, we can compute the value of an RPF with the following formula:
\[P(h_1, h_2) = Q(\alpha(h_1), \alpha(h_2)) \cdot e^{\ell(h_1) - \ell(h_2)}\]

The simplest RPF to represent in this form is the uniform RPF. In this case, there is a single possibility class so \(C = \{0\}\), and \(\ell\) is constant so it can be set as \(\ell(h) = 0\). No data is needed for Q because it has an identity rule where \(Q(0, 0) = 1\). In fact, Q is itself an RPF on C where instead of providing values in \(\mathbb{M}^{\ast}\) it provides values in the subspace \(\{\infty, 0, 1, \ast\}\).

\section{Topology and Limits in Relative Probability Space}
\label{section:topology}

One of the benefits of relative probability spaces is their properties with respect to limits. This section will show that limits of totally comparable RPFs are also totally comparable RPFs. In other words, the space of RPFs is topologically closed.

This effort caps off a significant argument in favor of relative probability. RPFs hold on to certain pieces information under the limit operation, while absolute probability does not. Some background in topology\footnote{Mendelson (1990) \cite{mendelson} and Bradley et al. (2020) \cite{bradley} the formal definitions and theorems which we will work off of here.} required for this section.

\subsection{Relative Probability Spaces}

\begin{definition}
\(\text{RPF}^{\ast}(K)\) is the set of relative probability functions of size K (where \(\Omega = \{0, 1, ..., K - 1\}\)). Likewise \(\text{RPF}(K)\) is the set of all totally comparable RPFs of size K.
\end{definition}

Because the set of absolute distributions with \(|\Omega| = K\) is embedded in \(\mathbb{R}^K\) as seen in figure \ref{fig:simplex}, its topological properties are well understood. The simplex is closed, bounded, and compact.

For relative probability distributions, there is no obvious way to embed it into K-dimensional euclidean space\footnote{Though it should be possible! See section \ref{section:euclidean_embedding}.}. The relative probability space is more complicated, because at the corners and edges of the simplex lurk entire subspaces where zero-probability outcomes are still being compared in different configurations.

Fortunately, \(\text{RPF}(K)\) can still be embedded into a much larger euclidean space. Any \(P \in \text{RPF}(K)\) is a function of type \(\Omega \times \Omega \rightarrow \mathbb{M}\) that satisfies the fundamental properties. Therefore \(\text{RPF}(K)\) can at least be embedded into \(\mathbb{M}^{K^2}\).

A topology can be defined by a \textit{basis of open sets}, and for euclidean space (and metric spaces more generally) the basis is the set of open balls on \(\mathbb{R}^n\).

\begin{definition}
An open ball of size \(\epsilon\) around point \(x \in \mathbb{R}^n\) is the set of all points y such that \(|x - y| < \epsilon\).
\end{definition}

Open balls do not work as a basis on the set \(\mathbb{M}\) because an open ball around \(\infty\) would contain only \(\infty\) when we instead want them to also contain some interval \((x, \infty]\). This is remedied by using the following transformation between \(\mathbb{M}\) and \([0, 1]\). We take this transformation to be continuous, or topology conserving, so that the topology of \(\mathbb{M}\) is defined by the topology of \([0, 1]\).

\begin{definition}
The \textit{inverse odds transform} is the function \(\text{odds}^{-1}: \mathbb{M} \rightarrow [0, 1]\) with \(\text{odds}^{-1}(0) = 0\) and \(\text{odds}^{-1}(\infty) = 1\) defined by
\[\text{odds}^{-1}(x) = \frac{x}{x + 1}\]
\end{definition}

The inverse odds transformation establishes \(\mathbb{M}\) as topologically equivalent to the closed interval \([0, 1]\). Now \(\text{RPF}(K)\) can be embedded into a bounded region of euclidean space, namely \([0, 1]^{K^2}\) by applying inverse odds to it. The result won't follow the fundamental axioms, but it will respect all of the topological properties. Appendix \ref{appedix:topology} provides an alternative way to define a topology on \(\text{RPF}(K)\).

\subsection{Limit Points and Compactness}

We now present an argument for the closure of \(\text{RPF}(K)\) under limits.

\begin{definition}
\(x\) is a \textit{limit point} of set \(A\) if any open set containing \(x\) also contains points in \(A\). Equivalently for euclidean space, any open ball containing \(x\) also intersects with \(A\).
\end{definition}

\begin{theorem}
\label{thm:rpf_closed}
Let \(P\) be a \textit{limit point} of \(\text{RPF}(K)\). Then \(P\) satisfies all of the fundamental axioms and is therefore a member of \(\text{RPF}(K)\). In other words, \(\text{RPF}(K)\) is topologically closed.
\end{theorem}

\begin{proof}
Let \(P\) be a limit point of \(\text{RPF}(K)\). We can show that for each fundamental axiom, if \(P\) doesn't satisfy the axiom then some open ball around \(P\) also doesn't satisfy the axiom.

Identity: If \(P\) doesn't satisfy the identity axiom, then \(P(h, h) \neq 1\) for some outcome \(h\). The inverse odds transform maps the value \(1\) to \(\frac{1}{2}\), so a choice of \(\epsilon\) less than \(|\text{odds}^{-1}(P(h, h)) - \frac{1}{2}|\) will contain only elements \(P'\) where \(P'(h, h) \neq 1\).

Inverse: A similar argument applies here if for some pair \(h_1\) and \(h_2\), \(P(h_1, h_2) \neq P(h_1, h_2)\). An \(\epsilon\) can be selected that is small enough so that \(P(h_1, h_2)\) is never equal to \(P(h_1, h_2)\)

Composition: Suppose \(P(a, c) :\ncong P(a, b) \cdot P(b, c) \). This can only be true if \(P(a, c) \neq P(a, b) \cdot P(b, c) \). But then this gap allows us to make the same argument from before; for some sufficiently small \(\epsilon\)-ball around \(P\), the composition axiom will still be false.

This means that only functions \(P\) which satisfy the fundamental axioms can be limit points to \(\text{RPF}(K)\), and therefore \(\text{RPF}(K)\) contains all its limit points and is closed.
\end{proof}

We can now prove that \(\text{RPF}(K)\) is compact thanks to the Heine-Borel theorem, stated as follows (wording from Bradley \cite{bradley}).

\begin{theorem}[Heine-Borel Theorem]
A subset of \(\mathbb{R}^n\) is compact if and only if it is closed and bounded.
\end{theorem}

\begin{theorem}
\(\text{RPF}(K)\) is compact.
\end{theorem}

\begin{proof}
Theorem \ref{thm:rpf_closed} states that \(\text{RPF}(K)\) is closed, and by the odds transform it is also bounded. By the Heine-Borel theorem, it is compact.
\end{proof}

\section{Future Work}
\label{section:future_work}
\subsection{Expansions to infinite spaces}
The obvious extension to this work is to expand relative probability to a generalized space which may be infinite, and thus capture all of the variety of probability distributions that one might wish to define. Section \ref{section:outcomes_to_events} provides a good framework for this which would start by modifying statement \ref{eq:event_def_ratio_match} to ask for an additive property. The relative probability function would then become a \textit{relative probability distribution}.

This process raises certain questions.
\begin{enumerate}
\item Relative probability would provide an interesting basis for the theory of finitely additive probability measures, or charges\cite{charges}. If we do not require countable additivity and instead opt for the weaker finite additivity, this would allow for constructions like a fair countable lottery (also known as a De Finetti Lottery\cite{de_finetti}). In the interest of studying limits of RPFs, this would be desirable. Finite additivity may also be a better mathematical model for computational methods in probability, on which additions must always be finite.
\item If we derive a notion of probability density, then can these densities at a particular pair of events be used to compare the relative probability of those events? What specific properties of the relative probability distribution are required to make this work? Can we start with a density function and work upwards?
\end{enumerate}

It appears possible to use these ideas to create a unified version of the Hausdorff measure - which finds the size of an object given its dimension. Instead of considering it to be multiple measures, we can have a single measure where bounded sets of equal dimension are mutually possible, and smaller-dimensional objects are mutually impossible with respect to larger dimensional objects.

\subsection{Relationship to Category Theory}

Category theorists will recognize that an RPF describes a \textit{thin category} where any pair of objects have at most one morphism connecting them (per direction). The relative probability axioms can be analyzed and approached through the lens of category theory in order to learn more about them.

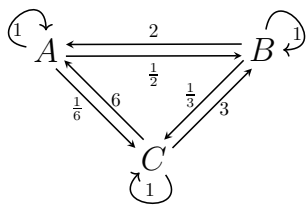
\begin{center}
\resizebox{0.3\textwidth}{!}{
\begin{tikzpicture} 
[node distance=0.8in]
\tikzset{
  shift left/.style ={commutative diagrams/shift left={#1}},
  shift right/.style={commutative diagrams/shift right={#1}},
  com/.style={circle,draw=none,inner sep=1pt,font=\LARGE}
}
\node (C) [com] {\(C\)};
\node (B) [com, above right = of C] {\(B\)};
\node (A) [com, above left = of C] {\(A\)};

\path[-stealth,thick,shift right=0.7ex] (C) edge node[right] {6} (A);
\path[-stealth,thick,shift right=0.7ex] (C) edge node[right] {3} (B);
\path[-stealth,thick,shift right=0.7ex] (B) edge node[above] {2} (A);

\path[-stealth,thick,shift right=0.7ex] (A) edge node[left,xshift=-0.6ex] {\(\frac{1}{6}\)} (C);
\path[-stealth,thick,shift right=0.7ex] (B) edge node[left,yshift=0.6ex] {\(\frac{1}{3}\)} (C);
\path[-stealth,thick,shift right=0.7ex] (A) edge node[below] {\(\frac{1}{2}\)} (B);

\path[<-,every loop/.style={looseness=5},thick] (A)
         edge  [in=170,out=90,loop,below] node {1} (); 
\path[<-,every loop/.style={looseness=5},thick] (B)
         edge  [in=80,out=0,loop,below] node {1} (); 
\path[<-,every loop/.style={looseness=9},thick,rotate=180] (C)
         edge  [in=-235,out=-315,loop,above] node {1} ();

\end{tikzpicture}
}
\end{center}

The recent work of Censi et al.\cite{censi} concerns negative information in categories, which corresponds to the wildcard element \(\ast\). It represents regions of the probability function that remain incomparable. This work could be used to subsume, develop, and refine the indeterminate wildcard concept.

\subsection{Embedding in Lower Dimensional Euclidean Space}
\label{section:euclidean_embedding}

Absolute probability functions have this advantage where they can be embedded into a simplex in \(\mathbb{R}^K\). For relative probability functions it is not so straightforward but it should still be possible given that \(\text{RPF}(K)\) is a (K-1)-dimensional set that can be embedded into \([0, 1]^{K^2}\). For example, the space \(\text{RPF}(3)\) can be mapped as a hexagon, where each point can be assigned a probability based on its distance between two parallel sides as in figure \ref{fig:hexagon}.

\begin{figure}[h]
\begin{center}
\resizebox{0.25\textwidth}{!}{
\begin{tikzpicture}
\node[regular polygon,draw,minimum size=4cm,line width=0.5mm,regular polygon sides = 6] (p) at (0,0) {};

\node at (p.corner 1) [anchor=360/6*(1-1)+270] {};
\node at (p.corner 2) [anchor=360/6*(2-1)+270] {};
\node at (p.corner 3) [anchor=360/6*(3-1)+270] {};
\node at (p.corner 4) [anchor=360/6*(4-1)+270] {};
\node at (p.corner 5) [anchor=360/6*(5-1)+270] {};
\node at (p.corner 6) [anchor=360/6*(6-1)+270] {};

\node[regular polygon,draw=none,rotate=30,minimum size=3.4cm,line width=0.5mm,regular polygon sides = 6] (A) at (0,0) {};
\node at (A.corner 1) [anchor=360/6*(1-1)+270] {\(P(a)=1\)};
\node at (A.corner 2) [anchor=360/6*(2-1)+230,rotate=58] {\(P(c)=0\)};
\node at (A.corner 3) [anchor=360/6*(4-1)+260,rotate=-60] {\(P(b)=1\)};
\node at (A.corner 4) [anchor=360/6*(4-1)+270] {\(P(a)=0\)};
\node at (A.corner 5) [anchor=360/6*(5-1)+230,rotate=55] {\(P(c)=1\)};
\node at (A.corner 6) [anchor=360/6*(7-1)+260,rotate=-60] {\(P(b)=0\)};

\draw[fill=black] (0.5,0.6) circle (3pt);

\draw[line width=0.5mm,dashed] (0.5,0.6) -- ($(p.corner 1)!(0.5,0.6)!(p.corner 6)$);
\draw[line width=0.5mm,dashed] (0.5,0.6) -- ($(p.corner 2)!(0.5,0.6)!(p.corner 3)$);
\draw[line width=0.5mm,dashed] (0.5,0.6) -- ($(p.corner 5)!(0.5,0.6)!(p.corner 4)$);
\end{tikzpicture}
}
\end{center}
\caption{A triangle probability simplex embedded into \(\mathbb{R}^2\) as a hexagon.}
\label{fig:hexagon}
\end{figure}
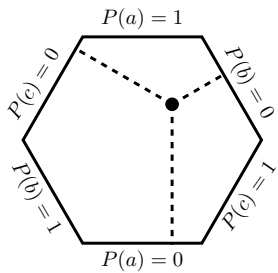

In this case, the triangular simplex has been truncated. For higher order simplices, this appears to become increasingly unwieldy unless some simplifying trick is developed. If it is successfully done, then a cleaner representation for members of \(\text{RPF}(K)\) becomes available.

\section{Conclusion}

Modeling probability as a ``relative'' measure takes on a new significance given the proliferation of computational methods. The nice properties of RPFs include formula simplification, closure under limits, and the ability to model non-standard distributions.

Should practitioners regularly use the language of relative probability when describing and deploying probabilistic models? It may not be beneficial in every situation, but they should certainly consider it! Computational techniques for Bayesian inference and machine learning in many cases remove the need to normalize the distributions involved, so why not make it official?

\begin{appendices}

\section{Topology and Limits in Relative Probability Space}
\label{appedix:topology}

In section \ref{section:topology}, we established a topology on \(\text{RPF}(K)\) by embedding it into Euclidean space, and demonstrated its compactness. Here is an alternative topology that defines a \textit{basis of open sets} directly. While this is not used in the context of this paper, it shows promise in providing a more natural framework for topological proofs and simplex partitioning.

The notion of an open set changes when a topological space is restricted to a lower dimension. For example, on the real number line \(\mathbb{R}\), the open interval (0, 1) is an open set. However, once it is embedded into \(\mathbb{R}^2\), it is now a line segment in a plane and no longer open (see figure \ref{fig:sub_topology}). For example, the set \(\{(x, y): x \in (0, 1)\;  \text{and}\;  y \in (-\epsilon, +\epsilon)\}\) given an \(\epsilon > 0\) is such an open set on \(\mathbb{R}^2\). The line segment (0, 1) is open in \(\mathbb{R}\) because it is the restriction of this open set on \(\mathbb{R}^2\). 

\begin{figure}[h]
\centering
\resizebox{0.4\textwidth}{!}{
\begin{tikzpicture} 
\draw[<->] (-0.7,0)--(5,0) node[right] {\(x\)};
\draw[<->] (0,-1,0)--(0,1) node[above] {\(y\)};
\draw[densely dashed,red,thick] (0,-0.5) rectangle (4.5,0.5);
\draw[draw=blue,thick] (0,0) circle (2pt) node[above left=-0.1,font=\ttfamily\tiny] {0};
\draw[draw=blue,thick] (4.5,0) circle (2pt) node [above right=-0.1,font=\ttfamily\tiny] {1};

\node[left=-0.05,font=\ttfamily\footnotesize] at (0,-0.5) {\(-\varepsilon\)};
\node[left=-0.05,font=\ttfamily\footnotesize] at (0,0.5) {\(\varepsilon\)};
\node[circle,inner sep=1pt] at (0,0) (0){};
\node[circle,inner sep=1pt] at (4.5,0) (1){};
\path[draw,blue,thick] (0)--(1);
\end{tikzpicture}
}
\caption{The small box that is the interior of the dotted rectangle is an open set in \(\mathbb{R}^2\), and therefore its restriction to \(\mathbb{R}\) - the line segment - is an open set in \(\mathbb{R}\). But the line segment is not open in \(\mathbb{R}^2\).}
\label{fig:sub_topology}
\end{figure}
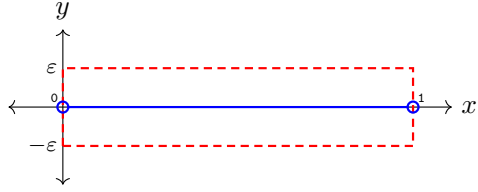

Likewise, an open set on \(\text{RPF}(K)\) will not be an open set on \(\text{RPF}(K+1)\). Starting with \(\text{RPF}(2)\), we find a totally comparable RPF that corresponds 1:1 with the magnitude space.

\begin{theorem}
Let \(\Omega = \{0, 1\}\) have two elements, with relative probability function \(P\). Then, \(P\) is completely determined by \(P(0, 1)\).
\end{theorem}

\begin{proof}
Let \(q = P(0, 1)\). By the inverse symmetric property, \(P(1, 0) = q^{-1}\). These values completely determine \(P\) on the outcome level.
\end{proof}

This gives us both a topology and a compactness proof for \(\text{RPF}(2)\) because it is isomorphic to \(\mathbb{M}\). Its basis for open sets are the open intervals of \(\mathbb{M}\), including those intervals that include 0 and \(\infty\). For \(K > 2\), we will need more powerful tools, namely \textit{open patches}. The set of open patches, which come in several flavors, will be a basis for the open sets on \(\text{RPF}(K)\).

\begin{definition}
An \textit{interior open patch} of \(\text{RPF}(K)\) is one of the following:

\begin{enumerate}
  \item If \(K = 2\), a subset parameterized by an interior open interval of magnitudes. \(\{P | a < P(h_1, h_2) < b\}\) for some \(a, b \in \mathbb{M}\) 
  \item If \(K > 2\), a composition of interior patches with composing function \(P_{\top}\) also being an interior patch.
\end{enumerate}
\end{definition}

Interior open patches contain only totally mutually possible functions as illustrated in figure \ref{fig:interior_open_patch}.

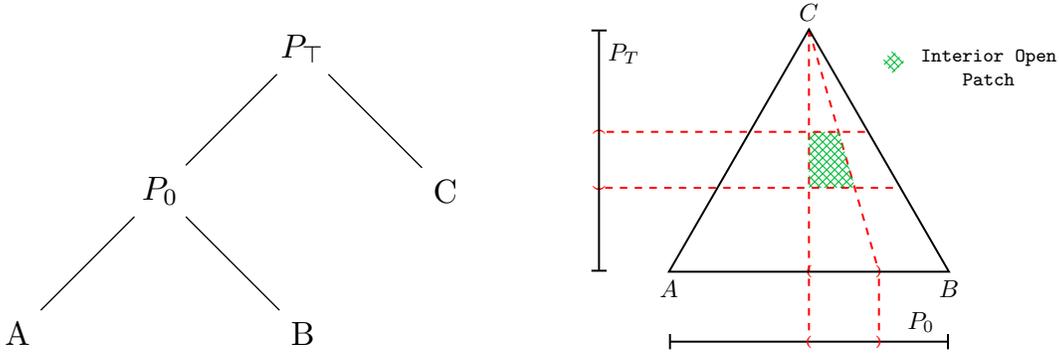
\begin{figure}[h]
\centering
\resizebox{0.4\textwidth}{!}{
\begin{tikzpicture}
\node {\(P_\top\)} [sibling distance = 3cm]
  child {node {\(P_0\)}  [sibling distance = 3cm]
    child {node {A}}
    child {node {B}}
  }
  child {node  {C}};
\end{tikzpicture}
}
 \hspace{3em}
\resizebox{0.4\textwidth}{!}{
\begin{tikzpicture} 
\draw[thick, dashed,red] (0,{2*sqrt(3)})--(1,0) node [font=\tiny\bf] {)};
\draw[thick, dashed,red] (0,{2*sqrt(3)})--(0,0) node [font=\tiny\bf] {(};
\path[draw=none,pattern=crosshatch, pattern color=darkgreen] (0,1.2)--(0.66,1.2)--(0.43,2)--(0,2)--cycle;

\node[rectangle,rotate=45,draw=none,pattern=crosshatch, pattern color=darkgreen] at (1.2,3) (P){};

\node(Q)[rectangle,draw=none,align=center,below right=-0.3 and 0.1 of P,font=\ttfamily\footnotesize] {Interior Open\\ Patch};

\draw[thick] (-2,0)--(2,0)--(0,{2*sqrt(3)})--cycle;

\node (A) at (-2,0) [below] {\(A\)}; 
\node (B) at (2,0) [below] {\(B\)}; 
\node (C) at (0,{2*sqrt(3)}) [above] {\(C\)}; 

\draw[thick, dashed,red] (0.8,2)--(-3,2) node[font=\tiny\bf,rotate=90] {)};
\draw[thick, dashed,red] (1.2,1.2)--(-3,1.2) node[font=\tiny\bf,rotate=90] {(};

\draw[thick, dashed,red] (1,-0.1)--(1,-0.9);
\draw[thick, dashed,red] (0,-0.1)--(0,-0.9);

\draw[thick,|-|] (-3,0)--(-3,{2*sqrt(3)}) node[right,pos=0.9] {\(P_{T}\)};

\draw[thick,|-|] (-2,-1)--(2,-1) node[above,pos=0.9] {\(P_{0}\)} node[red,font=\tiny\bf] at (1,-1) {)}
node[red,font=\tiny\bf] at (0,-1) {(};
\end{tikzpicture}
}
\caption{An interior open patch captures a contiguous set inside the probability simplex. Above is an example of a composite probability distribution where the diagram on the left shows how \(P_\top\) and \(P_0\) are composed, and the right shows how the interior segments on both interact to form a patch. }
\label{fig:interior_open_patch}
\end{figure}

\begin{definition}
A \textit{facet\footnote{A facet of a simplex is a subset where one parameter is equal to zero - equivalent to a face on a 3D object.} patch} of \(\text{RPF}(K)\) is one of the following:

\begin{enumerate}
  \item If \(K = 2\), an interval of the form \(\{P | 0 < P(h_1, h_2) < a\}\) for some \(a \in \mathbb{M}\) 
  \item If \(K > 2\), a composition where \(P_{\top}\) is drawn from an interior open patch, and all but one of the components are drawn from interior open patches. The final component - the \textit{facet component} - is itself drawn from a facet patch.
\end{enumerate}
\end{definition}

\begin{definition}
An \textit{exterior open patch} is a one of the following:

\begin{enumerate}
  \item A facet patch.
  \item A composition where \(P_{\top}\) is a facet patch. The \textit{facet component} is itself drawn from any open patch, and all the other components are drawn from interior open patches.
\end{enumerate}
\end{definition}

As seen in figure \ref{fig:exterior_open_patch}, exterior open patches touch the hyperfaces (facets) of the simplex as well as the vertices and edges.  As the number of dimensions increases and the composition diagram changes, more permutations are possible.

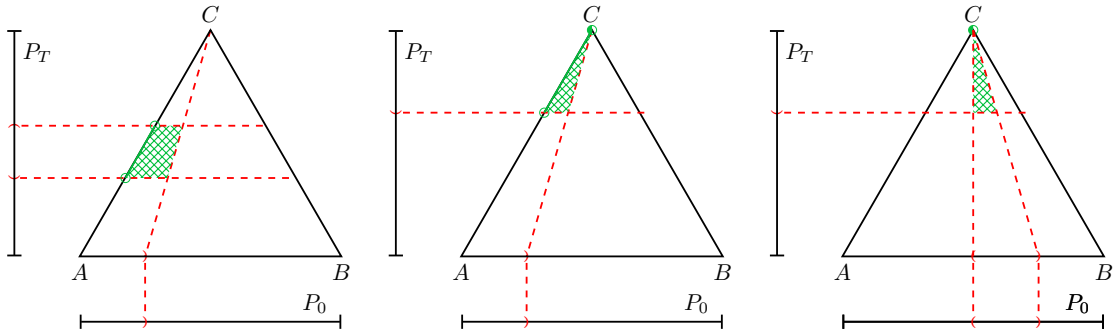
\begin{figure}[h]
\centering
\resizebox{0.3\textwidth}{!}{
\begin{tikzpicture} 
\draw[thick, dashed,red] (0,{2*sqrt(3)})--(-1,0) node [font=\tiny\bf] {)};
\path[draw=none,pattern=crosshatch, pattern color=darkgreen] (-1.3,1.2)--(-0.66,1.2)--(-0.43,2)--(-0.86,2)--cycle;

\draw[thick] (-2,0)--(2,0)--(0,{2*sqrt(3)})--cycle;

\node (A) at (-2,0) [below] {\(A\)}; 
\node (B) at (2,0) [below] {\(B\)}; 
\node (C) at (0,{2*sqrt(3)}) [above] {\(C\)}; 

\draw[thick, dashed,red] (0.8,2)--(-3,2) node[font=\tiny\bf,rotate=90] {)};
\draw[thick, dashed,red] (1.2,1.2)--(-3,1.2) node[font=\tiny\bf,rotate=90] {(};

\draw[thick, dashed,red] (-1,-0.1)--(-1,-0.9);
\draw[thick, darkgreen] (-1.3,1.2)--(-0.85,2);

\draw[draw=darkgreen] (-1.3,1.2) circle (2pt);
\draw[draw=darkgreen] (-0.85,2) circle (2pt);

\draw[thick,|-|] (-3,0)--(-3,{2*sqrt(3)}) node[right,pos=0.9] {\(P_{T}\)};

\draw[thick,|-|] (-2,-1)--(2,-1) node[above,pos=0.9] {\(P_{0}\)} node[red,font=\tiny\bf] at (-1,-1) {)};
\end{tikzpicture}
}
\resizebox{0.3\textwidth}{!}{
\begin{tikzpicture} 
\draw[thick, dashed,red] (0,{2*sqrt(3)})--(-1,0) node [font=\tiny\bf] {)};
\path[draw=none,pattern=crosshatch, pattern color=darkgreen] (-0.36,2.2)--(0,{2*sqrt(3)})--(-0.75,2.2)--cycle;

\draw[thick] (-2,0)--(2,0)--(0,{2*sqrt(3)})--cycle;

\node (A) at (-2,0) [below] {\(A\)}; 
\node (B) at (2,0) [below] {\(B\)}; 
\node (C) at (0,{2*sqrt(3)}) [above] {\(C\)}; 
\draw[thick, dashed,red] (0.8,2.2)--(-3,2.2) node[font=\tiny\bf,rotate=90] {(};
\draw[thick, dashed,red] (-1,-0.1)--(-1,-0.9);
\draw[thick, darkgreen] (-0.73,2.2)--(0,{2*sqrt(3)});

\draw[draw=darkgreen] (-0.73,2.2) circle (2pt);

\draw[draw=darkgreen] (0,{2*sqrt(3)}) circle (2pt);

\fill[darkgreen] (0,{2*sqrt(3)}) + (0, 2pt) arc (90:270:2pt);

\draw[thick,|-|] (-3,0)--(-3,{2*sqrt(3)}) node[right,pos=0.9] {\(P_{T}\)};

\draw[thick,|-|] (-2,-1)--(2,-1) node[above,pos=0.9] {\(P_{0}\)} node[red,font=\tiny\bf] at (-1,-1) {)};
\end{tikzpicture}
}
\resizebox{0.3\textwidth}{!}{
\begin{tikzpicture} 
\path[draw=none,pattern=crosshatch, pattern color=darkgreen] (0.36,2.2)--(0,{2*sqrt(3)})--(0,2.2)--cycle;

\draw[thick] (-2,0)--(2,0)--(0,{2*sqrt(3)})--cycle;

\node (A) at (-2,0) [below] {\(A\)}; 
\node (B) at (2,0) [below] {\(B\)}; 
\node (C) at (0,{2*sqrt(3)}) [above] {\(C\)}; 
\draw[thick, dashed,red] (0.8,2.2)--(-3,2.2) node[font=\tiny\bf,rotate=90] {(};

\draw[draw=darkgreen] (0,{2*sqrt(3)}) circle (2pt);

\fill[darkgreen] (0,{2*sqrt(3)}) + (0, 2pt) arc (90:270:2pt);

\draw[thick, dashed,red] (0,{2*sqrt(3)})--(1,0) node [font=\tiny\bf] {)};
\draw[thick, dashed,red] (0,{2*sqrt(3)})--(0,0) node [font=\tiny\bf] {(};

\draw[thick,|-|] (-3,0)--(-3,{2*sqrt(3)}) node[right,pos=0.9] {\(P_{T}\)};
\draw[thick,|-|] (-2,-1)--(2,-1) node[above,pos=0.9] {\(P_{0}\)};

\draw[thick, dashed,red] (1,-0.1)--(1,-0.9);
\draw[thick, dashed,red] (0,-0.1)--(0,-0.9);

\draw[thick,|-|] (-2,-1)--(2,-1) node[above,pos=0.9] {\(P_{0}\)}
node[red,font=\tiny\bf] at (1,-1) {)}
node[red,font=\tiny\bf] at (0,-1) {(};
\end{tikzpicture}
}
\caption{Exterior open patches. On the left is the facet patch, because it only touches a side (facet) of the simplex and not a corner. In the center is an exterior open patch where the facet component \(P_0\) is itself a facet patch (touching an edge and a corner), and on the right is an exterior open patch that touches a corner only because \(P_0\) is an interior open patch. Note that the point containing the corner at \(C\) in the middle and right diagram is only half filled because the patch contains some values where \(P(C) = 1\) and not others, depending on the relative probability between \(A\) and \(B\).}
\label{fig:exterior_open_patch}
\end{figure}

\begin{definition}
An \textit{open patch} is a subset of \(\text{RPF}(K)\) that is either an interior or exterior open patch.
\end{definition}

Now let the open patches be a basis for an open set thus defining a topology on \(\text{RPF}(K)\).

\begin{definition}
An \textit{open set} of \(\text{RPF}(K)\) is any (potentially infinite) union of open patches on \(\text{RPF}(K)\), or any finite intersection of open patches on \(\text{RPF}(K)\).
\end{definition}

These open patches are useful building blocks to building recursive proofs of compactness and other topological properties we might want to prove in the future.

\end{appendices}


\subsection*{Funding and Interests}

Local Maximum Labs is an ongoing effort to create and disseminate knowledge on intelligent computing, and its current resource is only the author's time. No funds, grants, or other support was received to assist with the preparation of this manuscript. The author certifies that they have no affiliations with or involvement in any external organization or entity with any financial interest or non-financial interest in the subject matter or materials discussed in this manuscript.

\end{document}